\documentclass{article} %
\usepackage{iclr2026_conference,times}

\usepackage{amsmath,amsfonts,bm}

\def\eqref#1{equation~\ref{#1}}

\def\1{\bm{1}}

\DeclareMathAlphabet{\mathsfit}{\encodingdefault}{\sfdefault}{m}{sl}
\SetMathAlphabet{\mathsfit}{bold}{\encodingdefault}{\sfdefault}{bx}{n}

\newcommand{\E}{\mathbb{E}}

\newcommand{\Var}{\mathrm{Var}}

\newcommand{\Cov}{\mathrm{Cov}}

\newcommand{\bs}{\mathbf{s}}

\usepackage[colorlinks]{hyperref}  
\hypersetup{
citecolor=[HTML]{3366CC}
}

\hypersetup{
linkcolor=[HTML]{3366CC}
}

\hypersetup{
urlcolor=black  
}

\usepackage{url}

\usepackage[utf8]{inputenc} %
\usepackage[T1]{fontenc}    %
\usepackage{url}            %
\usepackage{nicefrac}       %
\usepackage{microtype}      %
\usepackage{xcolor}         %
\usepackage{bbm}
\usepackage{pifont}

\usepackage{algpseudocode}
\usepackage[ruled,vlined,linesnumbered]{algorithm2e}

\usepackage{microtype}
\usepackage{graphicx}
\usepackage{subfigure}
\usepackage{booktabs} %
\usepackage{colortbl}

\usepackage{blindtext}
\usepackage{titlesec}

\usepackage{algorithm2e}
\usepackage{amsfonts}   %
\usepackage[table]{xcolor}
\usepackage[most]{tcolorbox}
\definecolor{impr}{RGB}{34, 139, 34}
\definecolor{lightred}{RGB}{255, 230, 230}
\definecolor{darkred}{RGB}{192, 0, 0}

\definecolor{bestbg}{HTML}{FFF2B2}   %
\definecolor{secondbg}{HTML}{E8F0FE} %
\definecolor{LavenderLight}{HTML}{C7C3F5}

\newcommand{\best}[1]{\cellcolor{bestbg}\textbf{#1}}
\newcommand{\sbest}[1]{\cellcolor{secondbg}\underline{#1}}

\newcommand{\highlightred}[1]{%
  \colorbox{lightred}{\parbox[t]{\linewidth}{\strut #1 \strut}}%
}

\renewcommand{\bibname}{References} 
\usepackage{enumitem}

\usepackage{amssymb}
\usepackage{mathtools}
\usepackage{amsthm}
\usepackage{amsmath}
\usepackage{multirow}
\usepackage{makecell}
\usepackage{tocbibind}
\usepackage{appendix}    
\usepackage{enumitem}
\usepackage{nccmath}
\usepackage{wrapfig} 
\definecolor{darkgreen}{rgb}{0.0, 0.5, 0.0} 
\usepackage[capitalize,noabbrev]{cleveref}

\theoremstyle{plain}
\newtheorem{theorem}{Theorem}[section]
\newtheorem{proposition}[theorem]{Proposition}
\newtheorem{lemma}[theorem]{Lemma}

\theoremstyle{definition}

\theoremstyle{remark}
\newtheorem{remark}[theorem]{Remark}

\usepackage{marvosym}
\usepackage{caption}
\usepackage[most]{tcolorbox}
\usepackage{amssymb}

\definecolor{LightBlue}{RGB}{173,216,230}
\newcommand{\our}[1]{\textsc{TaTToo}\xspace}

\newcommand{\circnum}[1]{%
  \tikz[baseline=(char.base)]{
    \node[shape=circle,fill=blue!30!white,inner sep=.6pt] (char)
    {\textcolor{white}{#1}};}
}

\tcolorboxenvironment{theorem}{
    colback=LightBlue!30,       %
    colframe=black!,     %
    arc=1.5mm,               %
    top=.5mm,
    bottom=.5mm,
    left=0mm,
    right=0mm,
    boxrule=.5pt,           %
    fonttitle=\bfseries,   %
    left=0mm,              %
    attach boxed title to top left={yshift=-2mm, xshift=4mm},
    boxed title style={
        colback=black!40,
        arc=2mm,
    }
}

\title{\our{}: Tool-Grounded Thinking PRM\\ for Test-Time Scaling in Tabular Reasoning}

\author{Jiaru Zou\textsuperscript{1,2}\thanks{Contact: \href{mailto:jiaruz2@illinois.edu}{jiaruz2@illinois.edu}} , Soumya Roy\textsuperscript{2}, Vinay Kumar Verma\textsuperscript{2}, Ziyi Wang\textsuperscript{3}, David Wipf\textsuperscript{2},\\ \textbf{Pan Lu\textsuperscript{4}}, \textbf{Sumit Negi\textsuperscript{2}}, \textbf{James Zou\textsuperscript{4}}, \textbf{Jingrui He\textsuperscript{1}}\\ 
$^1$UIUC, $^2$Amazon, $^3$Purdue University, $^4$Stanford University \\
}

\iclrfinalcopy %

\begin{document}

\maketitle

\vspace{-10pt}
\begin{abstract}

Process Reward Models (PRMs) have recently emerged as a powerful framework for enhancing the reasoning capabilities of large reasoning models (LRMs), particularly in the context of test-time scaling (TTS).
However, their potential for supervising LRMs on tabular reasoning domains remains underexplored. 
Through detailed empirical analyses, we identify that existing PRMs, though widely adopted for supervising text-only reasoning steps, struggle with table-specific operations such as sub-table retrieval and schema interaction, leading to critical performance bottlenecks.
To address this limitation, we propose \textit{\our{}}, a novel table-grounded PRM framework that (i) reasons explicitly over tabular reasoning steps and (ii) integrates tool-based verification to provide precise reward supervision. 
Concretely, we first design a scalable data curation pipeline that constructs over 60k high-quality step-level annotations by integrating table verification rationales with tool-based executions. 
Building on the collected data, we train \our{} with a dual-stage paradigm: cold-start supervised fine-tuning to capture tool-use reasoning patterns, followed by reinforcement learning with tool-grounded reward shaping to align our model with table-based verification.
We provide a comprehensive evaluation of the policy improvement induced by our newly designed PRM.
Across 5 challenging tabular reasoning benchmarks covering numerical reasoning, fact-checking, and data analysis, \our{} improves downstream policy LRMs by 30.9\% at inference, surpasses strong PRM baselines such as Qwen-2.5-Math-PRM-72B with only 8B parameters, and demonstrates strong generalizability across diverse TTS strategies.

\end{abstract}

\addtocontents{toc}{\protect\setcounter{tocdepth}{-1}}
\vspace{-2pt}
\section{Introduction}
\vspace{-2pt}

Tabular reasoning has become a fundamental capability for emerging large reasoning models (LRMs) across various real-world applications, including numerical analysis \citep{akhtar2023exploring, table_meet_llm}, fact-checking \citep{chen2019tabfact, parikh2020totto}, and question answering \citep{vakulenko2017tableqa, li2023table}.
Unlike free-form text, tables encode information in rows and columns with an implicit relational semi-structure.
Effective reasoning over tables therefore requires both accurate interpretation of tabular content and step-by-step logical inference to produce precise answers~\citep{chain_of_table, zhang2025survey}. 
To support such multi-step reasoning, recent studies such as Table-R1 series \citep{wu2025table, yang2025tabler1inferencetimescalingtable, jin2025table} have incorporated reinforcement learning (RL) techniques \citep{ppo,grpo} to better align LRMs with the demands of complex table understanding and reasoning.

On the other hand, process reward models (PRMs)~\citep{setlur2024rewarding,wang2024mathshepherdverifyreinforcellms,qwen2.5mathprm72b} have been developed to provide step-level supervision over model reasoning trajectories during test-time scaling (TTS), offering fine-grained verification that enhances LRMs’ performance at inference. However, despite growing computational budgets and increasing emphasis on advancing LRMs’ tabular reasoning abilities~\citep{ye2025limo,muennighoff2025s1simpletesttimescaling}, a corresponding step-level PRM to supervise the reasoning quality of these models in table domains is equally important but remains notably absent.
This gap motivates our study of a fundamental question:
\begin{tcolorbox}[
    enhanced,
    colback=yellow!10!white,      %
    colframe=black,            %
    coltitle=white,            %
    fonttitle=\bfseries,       %
    boxrule=.7pt,
    width=\textwidth,
    top=1mm,
    bottom=1mm,
    left=1mm,                  %
    right=1mm,                 %
    before skip=6pt, after skip=6pt,
    attach boxed title to top left={
        yshift=-2mm,
        xshift=2mm
    },
    boxed title style={
        colback=black,
        sharp corners,
        boxrule=0pt,
        top=2pt, bottom=2pt, left=4pt, right=4pt
    }，
]
\begin{center}
    \textbf{\raisebox{-0.2\height}{\includegraphics[height=1.3em]{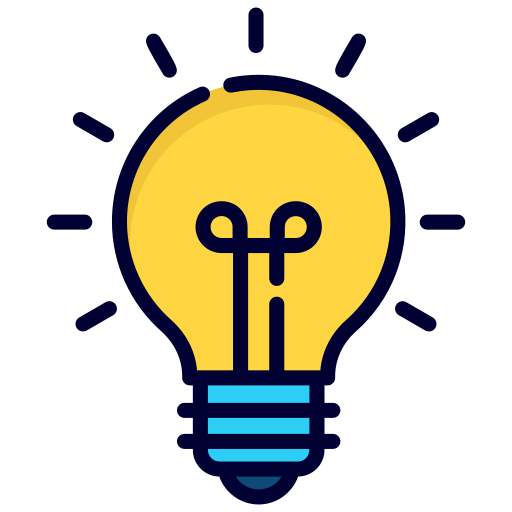}}%
    \ \textit{How can we provide reliable step-level supervision to advanced LRMs in tabular reasoning?}}
\end{center}

\end{tcolorbox}

To investigate this question, we first revisit several general-domain advanced PRMs and evaluate their effectiveness in supervising table-involved reasoning steps generated by LRMs.
Our analysis reveals that existing PRMs struggle to reliably verify two critical types of tabular CoT steps:
\textit{\ding{172} Table Retrieval}, where PRMs fail to supervise whether LRMs extract the correct sub-region of the input table relevant to the query; and \textit{\ding{173} Schema Interaction}, where PRMs cannot detect attention collapse \citep{dong2021attention}, as LRMs often overlook long-range table dependencies due to inherent locality bias. 
Beyond challenges arising from the tabular input modality, we also observe that current PRMs frequently introduce supervision errors within their own evaluation process, stemming from inaccurate table lookups or failed operations on tables. These shortcomings amplify bias and noise during TTS, ultimately creating persistent performance bottlenecks.

Motivated by our preliminary analyses, we propose \textbf{\our{}}, a new \textbf{\underline{Ta}}ble \textbf{\underline{T}}hinking PRM with \textbf{\underline{Too}}l integration abilities to provide more reliable and precise supervision for tabular reasoning.
Distinct from prior PRMs that provide weak supervision over table-specific operations, \our{} provides step-level supervision tailored to different input steps, applying both table-grounded rewards for tabular operation steps and inner-reasoning rewards for text-based reasoning steps. 
In addition, \our{} can leverage several external tools to interact with table contents, execute code-based operations, and incorporate the results back into the step-by-step verification process.
To build \our{}, we first design a scalable data curation pipeline that yields over 60k high-quality supervision instances by integrating expert verification rationales with tool-based executions. We then train our PRM under a dual-stage paradigm: supervised fine-tuning to capture step-level tool-use reasoning patterns, followed by reinforcement learning with a newly designed reward shaping scheme to encourage effective tool manipulation and faithful reasoning for accurate verification. Finally, we provide theoretical intuition on the policy improvement induced by incorporating \our{} during inference.

To demonstrate the effectiveness of \our{}, we conduct extensive experiments on five challenging tabular reasoning benchmarks, covering table-based question answering, numerical reasoning, fact-checking, and data analysis.
Across all benchmarks, incorporating 8B-size \our{} improves downstream policy models by 30.9\%.
In addition, \our{} consistently outperforms strong PRM baselines such as Qwen-2.5-Math-PRM-72B \citep{zhang2025lessons} and GenPRM-32B \citep{zhao2025genprmscalingtesttimecompute} with up to 9x parameter efficiency. 
In-depth analyses further demonstrate that incorporating our dual-stage training paradigm yields a 10.2\% improvement over standard PRM training, and \our{} exhibits strong generalizability across diverse TTS strategies, including Beam Search and DVTS.

\vspace{-5pt}
\section{Preliminary}
\label{sec:preliminary}
\vspace{-5pt}

\textbf{Table Understanding with LRMs.} 
We denote $T = (H, R)$ as a semi-structured table, where $H$ is the set of column headers defining the schema-level semantics, and $R$ is the set of rows, with each row composed of cell entries aligned with $H$.
Given a table $T$ and an associated natural language query $q$, we define a reasoning model as a conditional generation policy
$
\pi(\tau \mid T, q),
$
where $\tau = \{a_1, \dots, a_L\}$. Here, $\tau$ denotes the reasoning model's generated reasoning trajectory, including both intermediate reasoning steps $\{a_i\}^{L-1}_{i=1}$ and the final answer $a_L$.
In our problem setup, the intermediate reasoning steps consist of both model inner-thinking reasoning traces and tool-integrated programs that operate directly on the table to retrieve or compute intermediate results.
The final answer can take different formats depending on the query type, including textual or numerical values, boolean outputs (e.g., True/False), or executable programs (e.g., Python, SQL).

\textbf{Reward Modeling for Tabular Reasoning.}
Given a table $T$, a query $q$, and a candidate response $\tau$ generated by a policy LRM, a standard step-level verifier (i.e., PRM) parameterized by $\theta$ computes a scoring function $\mathcal{R}_\theta(\cdot)$ that assigns step-level rewards $r_i$ evaluating the correctness of each step $a_i \in \tau$. 
The trajectory-level reward $r_\tau$ for each response $\tau$ is then obtained by aggregating these step-level rewards. Formally, we have:
\begin{equation}
\label{eqa:prm_reward}
r_i = \mathcal{R}_\theta(a_i \mid T, q, \tau_{<i}), \quad \text{with } r_\tau = \mathcal{A} (r_1, r_2, \cdots, r_L),
\end{equation}
where $\mathcal{A}(\cdot)$ denotes an aggregation function such as \textsc{Mean} and \textsc{Sum} \citep{liu2025can}. The rewards provided by the PRM $\mathcal{R}_\theta$ can be further leveraged by a test-time compute strategy $\phi$ (e.g., Best-of-N \citep{brown2024largelanguagemonkeysscaling}, Beam Search \citep{snell2024scaling}) to guide resampling, refinement, and candidate selection among the responses generated by the policy model.

\vspace{-5pt}
\section{Why Table Reasoning Requires Verifiers Beyond Current PRMs?}
\label{sec:analyses}
\vspace{-5pt}

\begin{figure}[!t]
    \vspace{-5pt}
    \centering
    \begin{minipage}{0.7\textwidth}
        \centering
        \includegraphics[width=\linewidth]{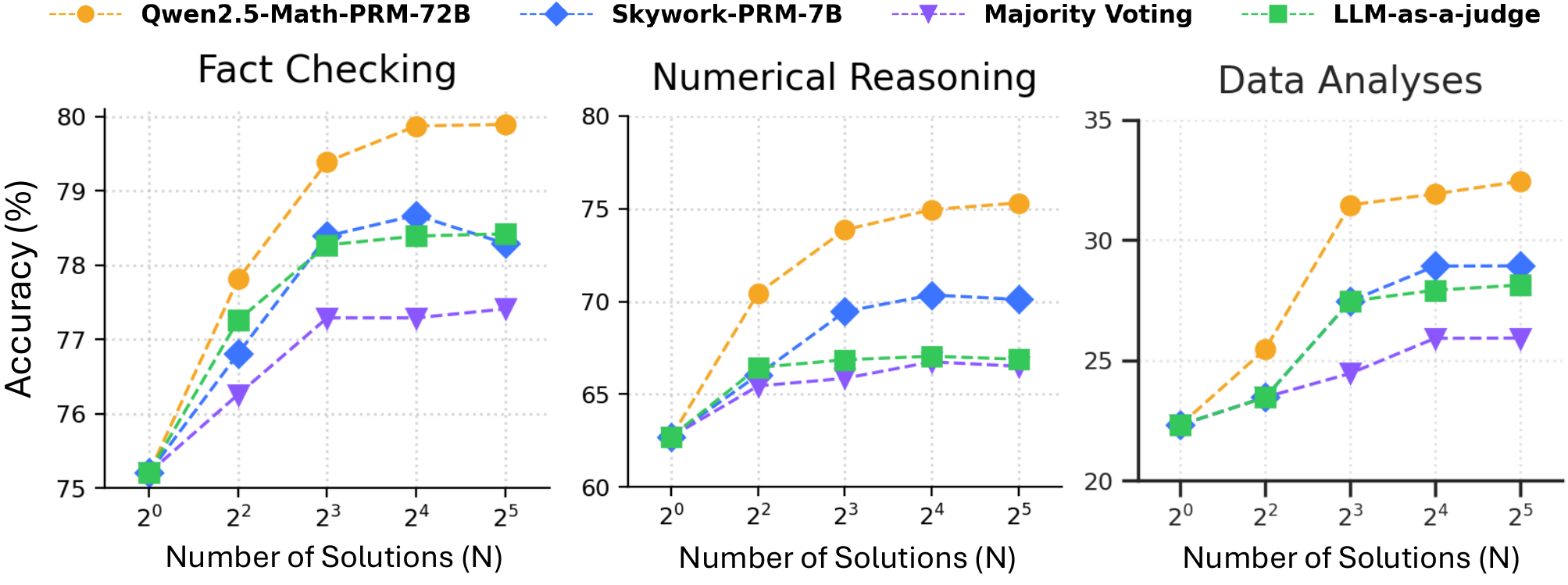}
        \caption{Best-of-N performance of DeepSeek-R1-Distill-Qwen-14B across 3 table tasks on TableBench with different types of step verifiers.
        }
        \label{fig:error_1}
    \end{minipage}
    \hfill
    \begin{minipage}{0.27\textwidth}
        \centering
        \includegraphics[width=\linewidth]{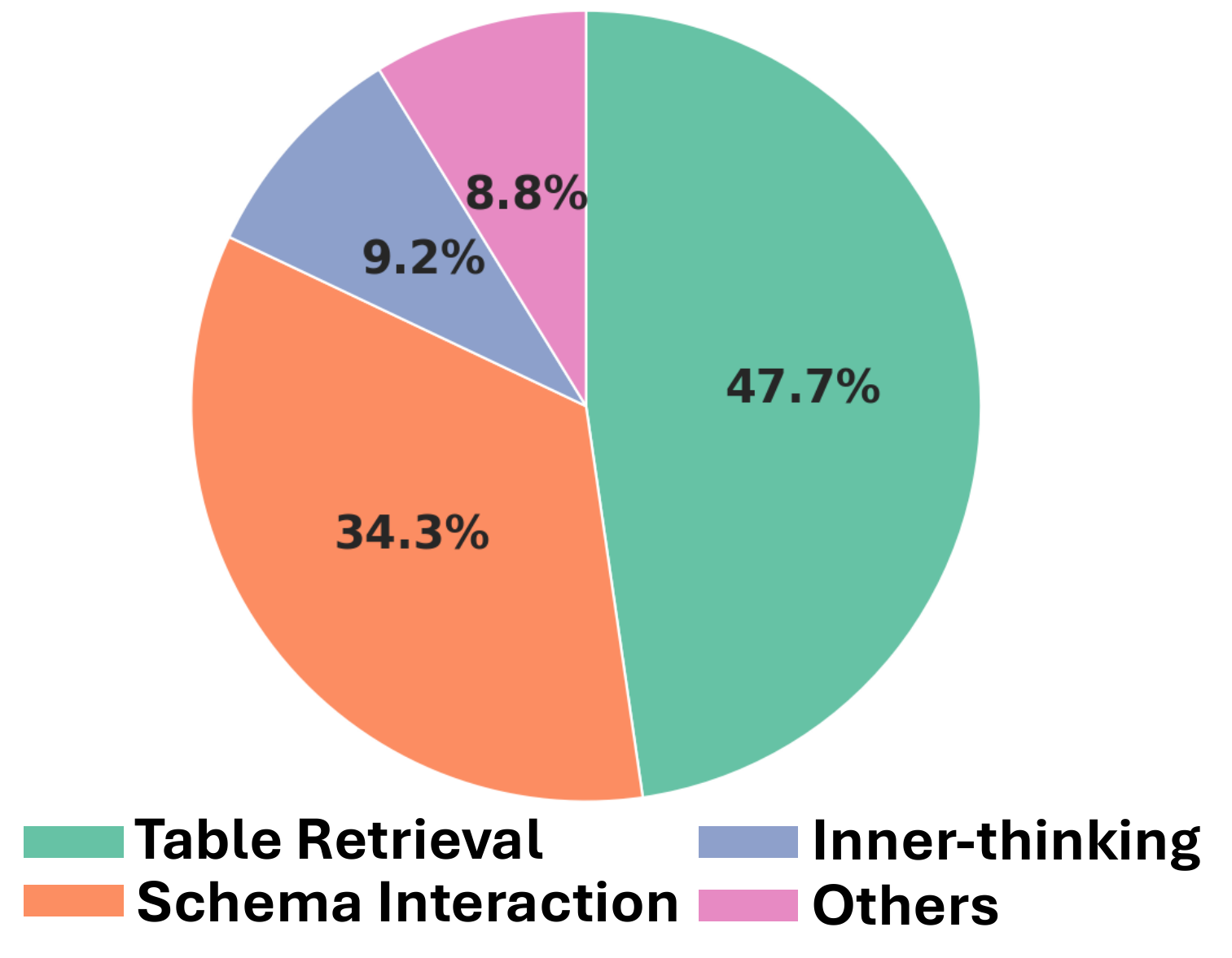}
        \caption{Error Distribution over 4 step categories across 500 incorrect cases after Best-of-N selection.}
        \label{fig:error_2}
    \end{minipage}
    \vspace{-10pt}
\end{figure}

We begin by revisiting existing general-domain PRM methods to assess their effectiveness in supervising LRMs on tabular reasoning tasks and to identify potential performance bottlenecks. To this end, we conduct a pilot study guided by two key questions:

\begin{tcolorbox}[
    enhanced,
    colback=yellow!10!white,      %
    colframe=black,            %
    coltitle=white,            %
    fonttitle=\bfseries,       %
    boxrule=.7pt,
    width=\textwidth,
    top=1mm,
    bottom=1mm,
    left=2mm,                  %
    right=2mm,                 %
    before skip=6pt, after skip=6pt,
    attach boxed title to top left={
        yshift=-2mm,
        xshift=2mm
    },
    boxed title style={
        colback=black,
        sharp corners,
        boxrule=0pt,
        top=2pt, bottom=2pt, left=4pt, right=4pt
    }，
]
\textit{\textbf{RQ1 -}} Beyond free-form text inputs, can common general-domain PRMs combined with TTS strategies also enhance the performance of LRMs on tabular reasoning tasks?

\vspace{3pt}

\textit{\textbf{RQ2 -}} When step-level reward supervision is crucial for tabular reasoning performance, how can PRMs effectively supervise and guide the quality of each reasoning step generated by LRMs?

\end{tcolorbox}

For brevity, we defer detailed experimental setups to Appendix \ref{app:set_up}. To investigate RQ1, we evaluate various step-level verification methods, including two advanced PRMs (Qwen2.5-Math-PRM-72B \citep{zhang2025lessons} and Skywork-PRM-7B \citep{skyworkopeno12024}), majority voting \citep{liu2025can}, and LLM-as-a-judge \citep{zheng2023judging} with the Best-of-N TTS strategy. We choose DeepSeek-R1-Distill-Qwen-14B \citep{deepseek_r1} as the common LRM and evaluate on TableBench \citep{tablebench}, which includes three fundamental table tasks (Fact Checking, Numerical Reasoning, and Data Analysis).
As shown in Figure~\ref{fig:error_1}, we observe that for small values of $N$, incorporating step-level verifiers into Best-of-$N$ generally improves LRM's performance over single-shot generation, with PRMs providing the largest gains.
However, once the number surpasses a threshold ($N \geq 8$), accuracy across all three table tasks converges to a bottleneck. For example, the performance of Qwen2.5-Math-PRM-72B on fact-checking is 79.19\%, 79.82\%, and 79.84\% for $N=\{8,16,32\}$, indicating that further increases in $N$ yield negligible gains, even though with PRM incorporation.

\begin{tcolorbox}[
    enhanced,
    colback=blue!5!white,      %
    colframe=black,            %
    coltitle=white,            %
    fonttitle=\bfseries,       %
    boxrule=0.5pt,
    width=\textwidth,
    top=1mm,
    bottom=1mm,
    left=2mm,                  %
    right=2mm,                 %
    before skip=6pt, after skip=6pt,
    attach boxed title to top left={
        yshift=-2mm,
        xshift=2mm
    },
    boxed title style={
        colback=black,
        sharp corners,
        boxrule=0pt,
        top=2pt, bottom=2pt, left=4pt, right=4pt
    }
]
\textbf{Observation 1 (Limitation on TTS):}  Existing PRMs yield modest improvements on tabular reasoning, but their efficacy quickly saturates, failing to fully exploit additional test-time compute.
\end{tcolorbox}

\begin{table}[!t]
    \centering
    \small
    \caption{Representative error cases in 3 different reasoning step categories. Each example highlights the erroneous model step in \textcolor{darkred}{red}, the corresponding error description, and the PRM’s (mis)judgment reward, illustrating where existing PRMs fail to detect mistakes.}
    \label{tab:error_example}
    \resizebox{\linewidth}{!}{%
    \begin{tabular}{l|>{\small}p{0.6\linewidth}|>{\small}p{0.5\linewidth}}
    \specialrule{1.2pt}{0pt}{0pt} 
    \textbf{Category} & \textbf{Wrong Model Answer} & \textbf{Error Description \& PRM output}\\
    \midrule
    \multirow{5}{*}{Table Retrieval} & 
    <think> Step 1: I need to figure out the average prominence of mountain peaks in the Democratic Republic of the Congo that have an elevation of at least 3000 meters. \newline \highlightred{Step 2: First, I'll review the provided table. My main focus is on the `country' and `elevation (m)' columns. \textcolor{darkred}{\textit{(Wrong Step)}}} \dots </think> \newline Final Answer is \texttt{\textbackslash boxed\{2306.65\}}
    & The model initially retrieves the wrong part of the table (columns \textit{`Country'} and \textit{`Elevation (m)'}), which leads to errors in all subsequent reasoning steps.
    \newline \newline \newline \newline
    \textbf{PRM $\rightarrow$ \{correct}\} \textcolor{red}{\ding{55} Fail to detect the error}
    \\
    \midrule
    \multirow{5}{*}{Schema Interaction} & 
    <think> Step 1: Let me start by looking at the table provided \dots The retrieved columns are \dots \newline \highlightred{Step 4: I'll go through each row and note down the episodes: Episode 1: Rating - 5.1, Episode 2: Rating - 5.2 \dots \textcolor{darkred}{\textit{(Wrong Step)}}} Step 5: Adding the rating together is 5.1 + 5.2 + \dots </think> \newline Final Answer is: \texttt{\textbackslash boxed\{11.52\}} 
    & The model identifies the correct column for the answer, but fails to extract all elements, omitting the last cell. This omission leads to incorrect reasoning in all subsequent steps.
    \newline \newline
    \textbf{PRM $\rightarrow$ \{correct}\} \textcolor{red}{\ding{55} Fail to detect the error}
    \\
    
    \midrule
    \multirow{5}{*}{Inner-thinking} & 
    <think> Step 1: Based on the table contents, I need to count podiums per season. 2008: $(1+1+1)/3 = 1$, \; 2009: $6$, \; 2010: $0$, \; 2011: $(2+0)/2 = 1$ \dots \newline \highlightred{Step 4: Then I will average across seasons. Final average is $(1+6+0+1)/4 = 2$  \textcolor{darkred}{\textit{(Wrong Step)}}} </think>  \newline Final Answer is: \texttt{\textbackslash boxed\{2\}}
    & The model incorrectly does the calculation by averaging the season-level means, giving each season equal weight, instead of averaging across all team-seasons. \newline \newline \newline \newline
    \textbf{PRM $\rightarrow$ \{incorrect}\} \textcolor{impr}{\ding{51} Detect the error}
    \\
    
    \specialrule{1.2pt}{0pt}{0pt} 
    \end{tabular}
    }
    \vspace{-5pt}
\end{table}

\textbf{Error Analysis.} Building on the observation, we further investigate the underlying causes of the performance bottleneck by conducting an error analysis on LRM's generation and PRM's supervision processes. 
Specifically, we sample 500 erroneous Best-of-N responses (N$=32$) selected by the PRM from LRM outputs, and ask human experts to classify them into 13 well-defined tabular error types (see Appendix~\ref{app:error_analysis}).
We then connect these errors with 4 reasoning-step categories reflecting the typical flow of LRMs’ reasoning process:
(i) \textit{Table Retrieval Steps}, locating relevant rows/columns regarding the input query; 
(ii) \textit{Schema Interaction Steps}, reasoning over the retrieved table contents, 
(iii) \textit{Inner-thinking Steps}, models' inner reasoning independent of table contents, 
and (iv) \textit{Others}, initial setup or final output steps that are irrelevant to core reasoning process.
Figure~\ref{fig:error_2} presents the error distribution across 4 reasoning step categories. We find that most errors arise in \textit{Table Retrieval} (47.7\%) and \textit{Schema Interaction} (34.3\%), implying that PRMs perform reasonably well on independent reasoning but fall short when reasoning steps involve table-specific operations. For better demonstration, we provide representative examples for each category in Table~\ref{tab:error_example}. 

\textbf{Why do PRMs fail on table-involved reasoning steps?} 
Next, we take a closer look at why PRMs lose their supervisory effectiveness when reasoning steps involve table operations.
For \textit{Table Retrieval Steps}, we conduct a contrastive experiment focusing particularly on the table contents retrieved by LRMs within their responses. We randomly sampled 500 responses and constructed two variants by (i) retaining the original LRM-retrieved sub-table, and (ii) replacing it with a randomly selected sub-table region from the original input table.
Figure~\ref{fig:prm_analyses} (left) shows the output rewards of Qwen2.5-Math-PRM-72B on both variants. The nearly identical distributions between real and random sub-tables indicate that current PRMs fail to distinguish retrieval correctness, suggesting that they are unable to assess whether the LRMs' retrieved portion of the table corresponds to the query.

\begin{tcolorbox}[
    enhanced,
    colback=blue!5!white,      %
    colframe=black,            %
    coltitle=white,            %
    fonttitle=\bfseries,       %
    boxrule=0.5pt,
    width=\textwidth,
    top=1mm,
    bottom=1mm,
    left=2mm,                  %
    right=2mm,                 %
    before skip=6pt, after skip=6pt,
    attach boxed title to top left={
        yshift=-2mm,
        xshift=2mm
    },
    boxed title style={
        colback=black,
        sharp corners,
        boxrule=0pt,
        top=2pt, bottom=2pt, left=4pt, right=4pt
    }
]
\label{takeaway:table_retrieval}
\textbf{Takeaway 1 (Table Retrieval):} Existing PRMs are insensitive to table retrieval correctness in the reasoning steps and fail to recognize whether the retrieved content corresponds to the query.
\end{tcolorbox}

\begin{figure}[!t]
    \vspace{-5pt}
    \centering
    \includegraphics[width=\linewidth]{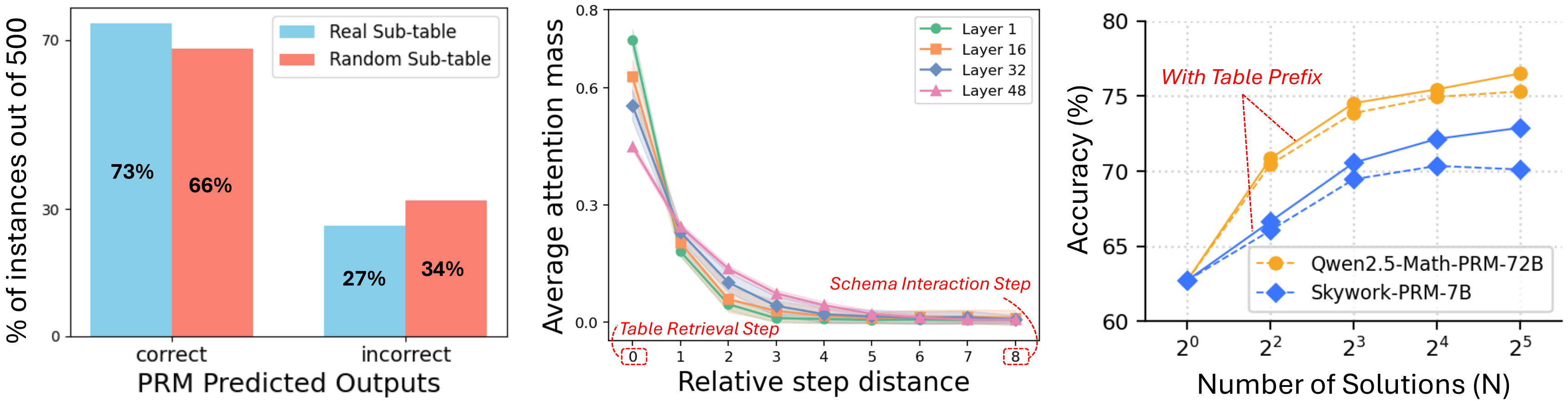}
    \caption{\textbf{Left:} PRM's rewards on 500 reasoning steps with the real-retrieved/randomly-replaced sub-table. \textbf{Middle:} Layer-wise average attention mass vs. relative step distance in tabular reasoning. Attention concentrates on nearby steps, with sharp decay as distance increases. \textbf{Right:} Best-of-N results on DeepSeek-R1-Distill-Qwen-14B for numerical reasoning with/without the table prefix.}
    \label{fig:prm_analyses}
    \vspace{-20pt}
\end{figure}

For \textit{Schema Interaction Steps}, we found in prior experiments that in the logic flow of LRMs' trajectories, table retrieval steps typically occur at the beginning, as the model must first extract relevant information from the table to answer the query. In contrast, schema interaction steps frequently occur far sentences from the beginning table retrieval steps, since LRMs tend to perform intermediate reasoning before revisiting their retrieved contents when needed.
Figure~\ref{fig:prm_analyses} (middle) illustrates the attention distribution of the LRM between the schema interaction step (step 8) and the table retrieval step (step 0). Due to the auto-regressive nature of LRMs, the schema interaction step attends primarily to nearby steps while assigning little attention to the earlier retrieval step. This inherent locality bias causes the model to frequently misinterpret or discard previously retrieved contents, even when the retrieval step has already extracted the correct information.
Moreover, current PRMs fail to supervise such misinterpretations, as their evaluations are highly localized to the current step rather than capturing dependencies on distant prior steps \citep{zou2025reasonflux,feng2025prm}.

\begin{tcolorbox}[
    enhanced,
    colback=blue!5!white,      %
    colframe=black,            %
    coltitle=white,            %
    fonttitle=\bfseries,       %
    boxrule=0.5pt,
    width=\textwidth,
    top=1mm,
    bottom=1mm,
    left=2mm,                  %
    right=2mm,                 %
    before skip=6pt, after skip=6pt,
    attach boxed title to top left={
        yshift=-2mm,
        xshift=2mm
    },
    boxed title style={
        colback=black,
        sharp corners,
        boxrule=0pt,
        top=2pt, bottom=2pt, left=4pt, right=4pt
    }
]
\textbf{Takeaway 2 (Schema Interaction):} Schema interaction steps under-attend to distant table retrieval contents due to locality bias. PRMs miss these failures since they can't look ahead and capture long-range dependencies among distant steps.
\end{tcolorbox}

\textbf{Table Prefix is the Key.} 
To explore potential solutions to the limitation above, we begin with a simple input modification for PRMs: prepending the retrieved table contents as a prefix to each schema interaction step. This grants PRMs direct access to the retrieval context, alleviating the need for long-range dependencies. 
We evaluate this modification and report the results in Figure~\ref{fig:prm_analyses} (right). 
Incorporating the table prefix indeed improves PRM supervision and leads to stronger downstream LRM performance. 
However, directly applying the prefix remains challenging, as current PRMs cannot automatically identify schema interaction steps, and the table prefixes obtained from LRMs are not guaranteed to be correct without proper supervision.

\textbf{Motivation for \our{}.} Our analyses above highlight the need for a principled step-level verifier capable of providing robust supervision over both table-grounded operations and models' inner reasoning. Motivated by this, we propose a new process reward model specifically designed to support LRMs in tabular reasoning.

\section{Building a Table-Grounded Step Verifier}
\label{sec:method}

We introduce \our{}, a generative PRM that provides reward supervision over both table operations and model inner thinking steps. Our method builds on two key components: 
(i) a large-scale data curation pipeline that synthesizes reasoning and tool usage for PRM training, and  
(ii) a dual-stage training paradigm that learns step-level verification with tool use optimization.

\subsection{Table-Aware and Tool-Integrated Supervision}
\label{sec:table_reward}

\textbf{Table-Aware Reward.} To align with the LRM's reasoning process on table tasks, we separate the supervision of table operations from model's inner reasoning part and decompose \our{}'s step-level reward (Eq.~\ref{eqa:prm_reward}) into two components:

\begin{equation}
\label{eqa:prm_reward_}
  r_i = 
  \begin{cases}
    r_{i,\text{rea}}, \text{ if } a_i \in \text{inner-thinking,}\\
    r_{i,\text{tab}}, \text{ if } a_i \in \text{table retrieval or schema interaction,}\\
  \end{cases}
  \text{and } r_\tau = \frac{1}{L}\sum_{i=1}^{L} r_i,
\end{equation}
where $r_{i,\text{rea}}$ captures the correctness of the model inner-reasoning process, $r_{i,\text{tab}}$ reflects the accuracy of table-grounded operations, and $r_\tau$ denotes the trajectory-level reward.

\textbf{Tool Integration in Verification.} A major limitation of current PRMs is their inability to supervise table-involved reasoning steps (as shown in Section~\ref{sec:analyses}). Meanwhile, recent studies~\citep{feng2025retool, qian2025toolrlrewardtoollearning} have shown that LLM agents can autonomously use \textbf{tools} to interact with external environments and iteratively refine their reasoning. In a similar spirit to address current PRM's limitation, we incorporate several external table-oriented tools into \our{}'s verification process to enable more reliable step supervision. We next describe how we curate a training set with tool-augmented, table-aware rewards and use it to train \our{}.

\subsection{\our{} Data Curation Pipline}
\label{sec:data_curation}

We design a large-scale data curation pipeline that simulates real-world scenarios of PRM tool use and step verification at scale. As illustrated in Figure \ref{fig:data_pipline}, there are three main stages:

\textbf{\circnum{1}Reasoning Trajectory Generation.} 
We begin by collecting trajectory responses from expert LRMs (e.g., DeepSeek-R1 \citep{deepseek_r1} and Claude-Opus-4.1 \citep{anthropic2025_claude_opus_sonnet_syscard}) on table-based questions drawn from diverse benchmarks, including TableInstruct~\citep{tablebench}, HybridQA~\citep{chen2020hybridqa}, ToTTo~\citep{parikh2020totto}, and WikiTQ~\citep{wikitq}.
We generate multiple responses per query and apply dual verification with human annotators and expert LLMs to filter out low-quality data, yielding a high-quality trajectory pool $\mathcal{T}_{\text{pool}}$ for subsequent labeling.

\textbf{\circnum{2}Verification Synthesis \& Reward Assignment.} 
We next provide step-level verification rationales and reward labels for each candidate response in $\mathcal{T}_{\text{pool}}$. 
(i) For \textit{table retrieval steps}, we extract the sub-table in each step and use LLM-as-a-judge to assess its relevance to the query, assigning table reward $r_{i,\text{tab}} \in \{-1,1\}$ based on retrieval correctness. 
(ii) For \textit{schema interaction steps}, we prepend the accurate sub-table as a table prefix to each collected verification rationale (according to our table-prefix analysis in Section~\ref{sec:analyses}) and assign $r_{i,\text{tab}} \in \{-1,1\}$ based on the correctness of the specific table-based operations or reasoning. (iii) For \textit{inner-thinking steps}, which involve no table contents, we apply LLM-as-a-judge and follow established labeling strategies~\citep{zhao2025genprmscalingtesttimecompute, khalifa2025processrewardmodelsthink} to assign $r_{i,\text{rea}} \in \{-1,1\}$ based on reasoning quality.

\begin{figure}[!t]
    \centering
    \includegraphics[width=\linewidth]{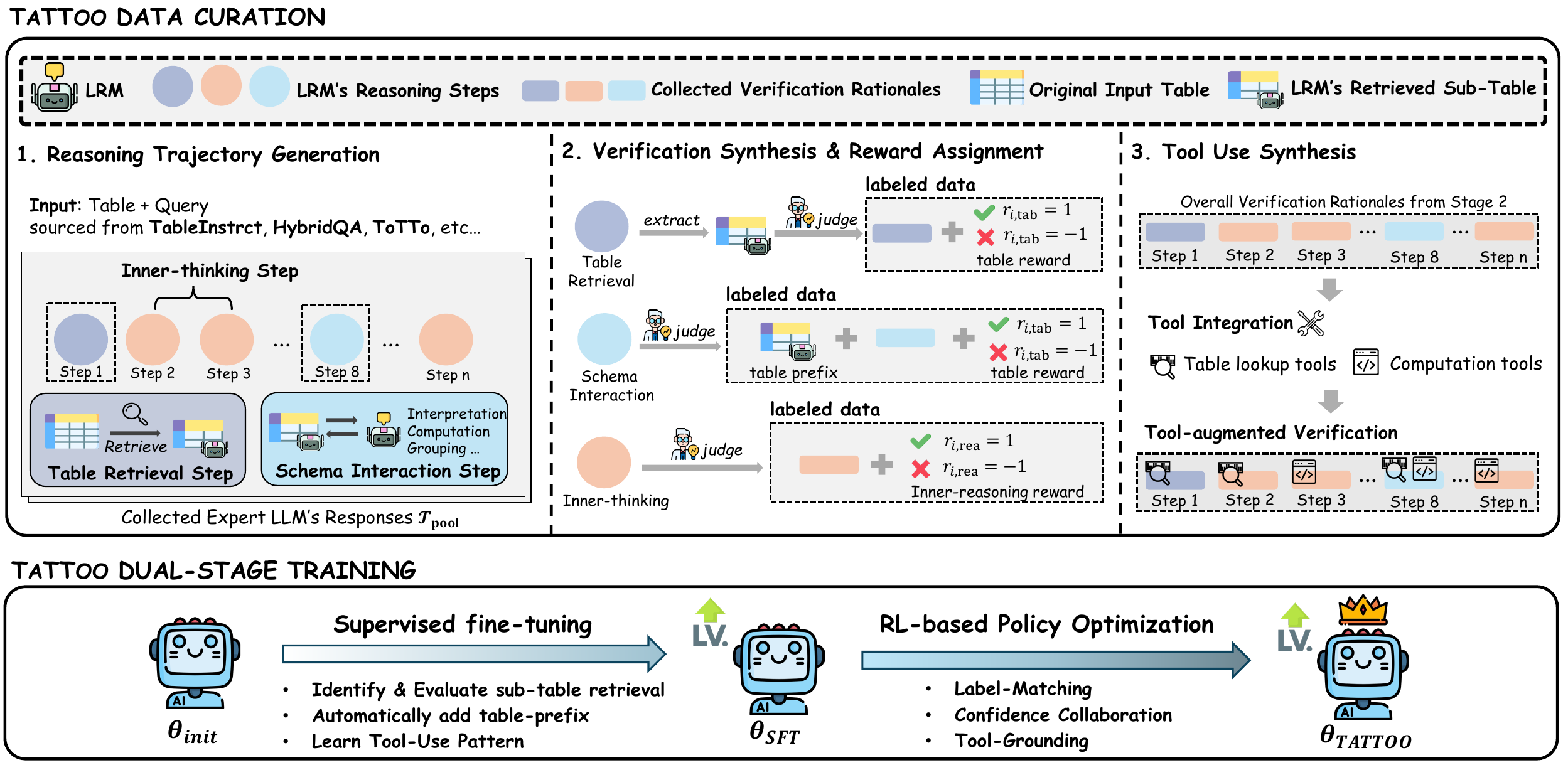}
    \caption{\textbf{Overview of \our{} framework.} We first curate 60k high-quality instances by collecting expert verification rationales with tool integration (Section~\ref{sec:data_curation}). We then train our PRM through a dual-stage training paradigm to achieve tool-grounded step-by-step reward supervision (Section~\ref{sec:training_PRM}).
    }
    \label{fig:data_pipline}
    \vspace{-5pt}
\end{figure}

\textbf{\circnum{3}Tool Use Synthesis.} 
To train \our{} to leverage tools for more accurate verification, we further augment the collected verification rationales with tool invocations, execution results, and feedback at the step level.
Specifically, inside the rationale contents, we replace manual reasoning for table lookups or calculations with the corresponding tool call and its execution output. We primarily employ two types of table tools:
(i) \textit{Computation tools:} code snippets (e.g., Python, SQL) for arithmetic and aggregation over table inputs;  
(ii) \textit{Table Lookup tools:} DataFrame APIs (e.g., Polars) or Lookup Utilities (e.g., CSV/Excel readers) for retrieving specific rows, columns, or cells during verification.  

Finally, we construct over 60k high-quality training instances with complete verification rationales and step-level rewards. This dataset is then used to train \our{} to integrate tool use with reasoning for robust step supervision. We leave additional data curation details in Appendix \ref{app:data_curation}.

\subsection{Tool-Grounded Dual-Stage Training}
\label{sec:training_PRM}

With the training data recipe in place, we train \our{} via a dual-stage paradigm: supervised fine-tuning to capture tool-integrated verification patterns, followed by RL-based policy optimization with a newly designed reward shaping scheme to further refine our PRM's step-level rationales and ensure accurate verification.

\textbf{Table-Aware Verification with Tools via SFT.}
We first finetune our PRM $\mathcal{R}_\theta$ on the curated dataset (Section~\ref{sec:data_curation}).
Specifically, given a training instance $(T, q, \tau)$ consisting of a table $T$, a query $q$, and an LRM-generated trajectory $\tau = (a_1,\dots,a_L)$, we train the PRM to output, for each step $a_i \in \tau$, a verification rationale $v_i$ together with its corresponding step-level reward $r_i$.
By formulating PRM training as language modeling, $\mathcal{R}_\theta$ is optimized auto-regressively to (i) identify accurate sub-table regions, (ii) learn to dynamically incorporate the retrieved table prefix into each schema interaction step, and (iii) generate verification rationales with tool-integration patterns.

\textbf{Tool-Grounded Reward Shaping in RL.} Prior generative PRM approaches~\citep{liu2025can, khalifa2025processrewardmodelsthink, zhao2025genprmscalingtesttimecompute} typically conclude PRM training after the SFT stage. In contrast, we draw inspiration from recent advances in agentic RL~\citep{o1, deepseek_r1} and further apply policy optimization to more tightly align the PRM’s verification process with effective tool utilization.
Specifically, we optimize $\mathcal{R}_\theta$ with a modified GRPO~\citep{grpo} by providing dense, tool-grounded supervision signals during policy optimization.
During RL rollouts of each training instance  $(T,q,\tau)$, we replace the original rule-based GRPO supervision signal with a denser per-step reward signal $s_i$, defined as:
\begin{equation}
\label{eqa:rewar_design}
s_i \;=\;
\underbrace{\mathbbm{1}\{\hat r_i = r_i\}}_{\text{\textbf{label-matching}}}
\;-\;
\underbrace{\lambda_{\mathrm{cal}}\Big(-\log \mathcal{R}_\theta(r_i \mid T,q,\tau)\Big)}_{\text{\textbf{confidence calibration}}}
\;+\;
\underbrace{\lambda_{\mathrm{tool}} \cdot \mathrm{support}(\hat v_i)}_{\text{\textbf{tool-grounding}}},
\end{equation}
where $\hat r_i$ is the PRM's predicted step-reward and $r_i$ is the ground-truth step-reward for the input step $a_i \in \tau$; $\hat v_i$ denotes the verification rationale generated by the PRM at step $i$, and $\mathrm{support}(\hat v_i) \in \{0,1\}$ measures whether the rationale correctly incorporates tool outputs; and  $\lambda_{\mathrm{cal}}$, $\lambda_{\mathrm{tool}}$ are tunable coefficients.
Besides enforcing correctness with the \textit{label-matching term}, the \textit{confidence calibration term} stabilizes training by encouraging higher probability on the ground-truth label, and the \textit{tool-grounding term} encourages rationales that effectively incorporate tool outputs.
We then aggregate the per-step signals $s_i$ into a trajectory-level training reward, normalize it within each sampled group to compute group-relative advantages, and update the PRM $\mathcal{R}_\theta$ under the GRPO objective.

\subsection{Inference-time Policy Improvement -- An Intuitive View}
To intuitively elucidate the role of \our{} and its table-aware rewards on LRM's tabular reasoning process (Eq. \ref{eqa:prm_reward_}), we provide a theoretical analysis on the policy improvement induced by \our{}.

Recall that the goal of our PRM is to improve the generated trajectory $\tau$ sampled from a policy LRM $\pi$, i.e., $\tau \sim \pi(\cdot \mid T,q)$. 
By combining the input table and query, we represent $(T,q,a_1,\dots,a_{i-1})$ as the current state $\mathbf{s}_i$. At step $i$, the policy LRM $\pi$ samples an action $a_i \sim \pi(\cdot \mid \mathbf{s}_i)$. 
We define the $Q$-value of policy $\pi$ as the expected future success, measured by the final answer $a_L$ correctness, i.e.,
\begin{equation}
    Q^{\pi}(\mathbf{s}_i, a_i) = Q^{\pi}\big((T,q, a_1, \dots, a_{i-1}), a_i\big) = \E_{a_{i+1}, \dots, a_{L} \sim \pi(\cdot \mid \mathbf{s}_i) }\left[\mathbbm{1}_{a_L \text{is correct}} \right].
\end{equation}
The value of policy $\pi$ at state $\mathbf{s}_i$ is defined as the expectation of $Q$-values under $\pi$'s next action distribution:
$V^{\pi}(\mathbf{s}_i) = \E_{a_i \sim \pi(\cdot \mid \mathbf{s}_i)} [Q^{\pi}(\mathbf{s}_i, a_i)]$. 
We now analyze the policy improvement afforded by \our{}'s table-aware reward $r_i$ supervision under one step of a natural policy gradient updating.

\begin{theorem}[\textbf{Policy Improvement (Lower Bound)}]
\label{theorem:improvement}
Given the current policy $\pi$, after one natural policy gradient update step guided by the PRM reward $r_i$ defined in Eq.\ref{eqa:prm_reward_}, we obtain the revised policy $\pi'(a_i \mid \mathbf{s}_i) \propto \exp(Q^{\pi}(\mathbf{s}_i, a_i) + r_i(\mathbf{s}_i, a_i)) $. 
The resulting expected policy improvement over the state distribution $\rho$ then satisfies:
{\small
\begin{align}
    \nonumber\E_{\mathbf{s}_i\sim\rho}  \left[V^{\pi'}(\mathbf{s}_i) - V^{\pi}(\mathbf{s}_i)\right] \gtrsim 
    \underbrace{\E_{\mathbf{s}_i\sim\rho}\Var_{a_i\sim \pi(\cdot\mid \mathbf{s}_i)}\left[r_{i, \text{tab}}(\mathbf{s}_i, a_i)\right]}_{\normalsize\textbf{\textcolor{red!60}{distinguishability from table reward $r_{i,\text{tab}}$}}} +
    \underbrace{\E_{\mathbf{s}_i\sim\rho}\Var_{a_i\sim \pi(\cdot\mid \mathbf{s}_i)}\left[r_{i,\text{rea}}(\mathbf{s}_i, a_i)\right]}_{\normalsize\textbf{distinguishability from inner-reasoning reward $r_{i,\text{rea}}$}} \\
   +\underbrace{\E_{\mathbf{s}_i\sim\rho}\E_{a_i\sim \pi(\cdot\mid \mathbf{s}_i)}\left[r_{i,\text{tab}}(\mathbf{s}_i, a_i)A^{\pi}(\mathbf{s}_i, a_i)\right]}_{\normalsize\textbf{\textcolor{red!60}{alignment between $r_{i, \text{tab}}$ and $A^{\pi}$}}}
   + \underbrace{\E_{\mathbf{s}_i\sim\rho}\E_{a_i\sim \pi(\cdot\mid \mathbf{s}_i)}\left[r_{i, \text{rea}}(\mathbf{s}_i, a_i)A^{\pi}(\mathbf{s}_i, a_i)\right]}_{\normalsize\textbf{alignment between $r_{i, \text{rea}}$ and $A^{\pi}$}},
\end{align}
}

where $A^{\pi}(\mathbf{s}_i, a_i) = Q^\pi(\mathbf{s}_i, a_i) - V^{\pi}(\mathbf{s}_i)$ denotes the advantage of policy $\pi$. 
\end{theorem}

Theorem~\ref{theorem:improvement} (proof in Appendix~\ref{app:proof}) explains that our decomposable reward design $r_i$ enables each component to additively contribute to policy improvement, provided that the reward components are each individually aligned with the policy advantage function. In this way, the table-aware rewards provided by \our{} help ensure targeted supervision on both inner reasoning and table-involved operations generated by LRMs. Below, we further empirically evaluate the effectiveness of \our{} across various downstream tabular reasoning tasks.

\vspace{-5pt}
\section{Empirical Evaluations}
\vspace{-5pt}
\label{sec:exp}

\textbf{Baselines and Models.} 
We compare \our{} against various types of step-level verification methods, including advanced PRMs, majority voting \citep{liu2025can}, and LLM-as-a-judge \citep{zheng2023judging}. The setups for these baselines are aligned with our preliminary analyses in Section~\ref{sec:analyses}. For PRM approaches, we include both discriminative (Qwen-PRM series \citep{zhang2025lessons}, Math-Shepherd-PRM \citep{wang2024mathshepherdverifyreinforcellms}, and Skywork-PRM \citep{skyworkopeno12024}) and generative (ThinkPRM \citep{khalifa2025processrewardmodelsthink} and GenPRM \citep{zhao2025genprmscalingtesttimecompute}). 
Regarding the policy reasoning models, we evaluate our proposed method on DeepSeek-R1-Distill-Qwen-14B \citep{deepseek_r1}.
Further details on the baselines and policy models setups are provided in Appendix~\ref{app:baseline_model}.

\textbf{Datasets.} 
We evaluate on four representative and challenging benchmarks spanning diverse tabular reasoning tasks, including 
(i) TableBench (TB) \citep{tablebench}, a complex tabular reasoning benchmark with 886 questions covering tasks of numerical reasoning (NR), fact checking (FC), and data analysis (DA).  
(ii) WTQ \citep{wikitq}, a benchmark for complex question answering over Wikipedia tables.    
(iii) MMQA \citep{wummqa}, a multi-table understanding benchmark covering table retrieval, multi-hop \& multi-table QA and text-to-SQL generation. We leave the additional dataset descriptions in Appendix \ref{app:dataset}.

\textbf{Implementation Details.} 
We train \our{} on the off-the-shelf Qwen-3-8B model \citep{qwen3} using our 60k curated training instances (Section~\ref{sec:data_curation}). All training and inference experiments are conducted on 8×A100-80G GPUs. To evaluate \our{} under different TTS strategies, we adopt three representative methods, including Best-of-N \citep{brown2024largelanguagemonkeysscaling}, Beam Search \citep{snell2024scaling}, and Diverse Verifier Tree Search (DVTS) \citep{beeching2024scalingtesttimecompute}. Additional implementation details on training setup and configurations of \our{} are provided in Appendix~\ref{app:implementation}.

\begin{table*}[!t]
    \centering
    \caption{Main results of \our{} on 5 different tabular reasoning tasks. We report the best-of-N (with $N = \{4,8,16,32\}$) performance using DeepSeek-R1-Distill-Qwen-14B as the policy model and compare against various step verifiers. The best and second-best results are highlighted. \our{} consistently achieves state-of-the-art TTS performance with significantly fewer parameters.}
    \label{tab:main_res}
    \resizebox{\textwidth}{!}{
    \begin{tabular}{l|c|cccc|cccc|cccc|cccc|cccc}
        \toprule
        \multirow{2}{*}{\textbf{Verifer (Best-of-N)}} & \multirow{2}{*}{\textbf{Params}} & \multicolumn{4}{c}{\textbf{TB-NR}} & \multicolumn{4}{c}{\textbf{TB-FC}} & \multicolumn{4}{c}{\textbf{TB-DA}} & \multicolumn{4}{c}{\textbf{WTQ}} & \multicolumn{4}{c}{\textbf{MMQA}} \\
        \cmidrule(lr){3-6} \cmidrule(lr){7-10} \cmidrule(lr){11-14} \cmidrule(lr){15-18} \cmidrule(lr){19-22}
        & & 4 & 8 & 16 & 32 & 4 & 8 & 16 & 32 & 4 & 8 & 16 & 32 & 4 & 8 & 16 & 32 & 4 & 8 & 16 & 32\\
        \midrule[0.35pt]
        Majority Vote & - &
        65.5 & 65.9 & 66.8 & 66.5 &
        76.2 & 77.3 & 77.3 & 77.4 &
        23.5 & 24.5 & 26.0 & 26.1 &
        64.7 & 65.3 & 67.3 & 67.0 &
        18.4 & 19.4 & 20.4 & 20.1 \\
        LLM-as-a-judge & - &
        66.7 & 66.9 & 67.1 & 66.9 &
        77.2 & 78.3 & 78.4 & 78.6 &
        23.5 & 27.4 & 28.0 & 28.4 &
        65.2 & 66.4 & 68.1 & 68.1 &
        19.6 & 21.3 & 22.5 & 22.7 \\
        Skywork-PRM-7B & 7B &
        66.1 & 69.5 & 70.3 & 70.1 &
        76.8 & 78.4 & 78.6 & 78.3 &
        24.1 & 27.5 & 28.9 & 29.1 &
        65.9 & 67.5 & 68.4 & 68.6 &
        21.4 & 24.6 & 25.1 & 25.3 \\
        Math-Shepherd-PRM-7B & 7B &
        67.2 & 70.6 & 71.5 & 71.8 &
        76.2 & 76.9 & 76.8 & 77.1 &
        22.7 & 24.8 & 26.4 & 25.9 &
        66.8 & 68.7 & 69.6 & 69.3 &
        22.0 & 25.2 & 25.9 & 26.1 \\
        Qwen2.5-Math-PRM-7B & 7B &
        66.9 & 70.1 & 71.7 & 72.5 &
        75.4 & 77.2 & 77.9 & 77.4 &
        23.2 & 25.4 & 26.3 & 26.6 &
        65.2 & 68.5 & 69.6 & 69.7 &
        23.5 & 25.2 & 27.1 & 27.3 \\
        ThinkPRM & 14B &
        69.2 & 70.7 & 73.5 & 73.8 &
        75.8 & 75.4 & 76.3 & 76.9 &
        21.6 & 22.7 & 23.1 & 22.8 &
        64.3 & 66.1 & 65.7 & 65.9 &
        22.4 & 22.7 & 23.6 & 23.0 \\
        GenPRM & 32B &
        \best{71.5} & 73.5 & 73.7 & 74.2 &
        76.3 & 78.5 & 79.2 & 79.4 &
        25.3 & 27.9 & 30.2 & 30.7 &
        \best{69.8} & \best{72.5} & \sbest{73.3} & \sbest{73.1} &
        23.8 & 25.4 & 26.2 & 26.4 \\
        Qwen2.5-Math-PRM-72B & 72B &
        70.4 & \sbest{73.8} & \sbest{74.9} & \sbest{75.3} &
        \best{77.8} & \sbest{79.2} & \sbest{79.8} & 79.8 &
        \sbest{25.5} & \sbest{31.5} & \sbest{32.0} & \sbest{32.4} &
        \sbest{69.2} & 71.8 & 73.0 & 72.6 &
        \sbest{24.4} & \sbest{26.8} & \sbest{28.7} &\sbest{28.6} \\
        \textbf{\our{}} & 8B &
        \sbest{71.2} & \best{74.2} & \best{76.4} & \best{78.1} &
        \sbest{77.4} & \best{79.6} & \best{81.2} & \best{82.0} &
        \best{27.7} & \best{31.9} & \best{33.6} & \best{34.3} &
        \best{69.8} & \sbest{72.3} & \best{73.5} & \best{74.9} & \best{25.1} & 
        \best{27.2} & \best{29.1} &
        \best{30.5} \\
        \bottomrule
    \end{tabular}
    }
    \vspace{-10pt}
\end{table*}

\subsection{Main Results}

Table~\ref{tab:main_res} reports the Best-of-N performance of incorporating \our{} on the DeepSeek-R1-Distill-Qwen-14B model across five tabular reasoning tasks. 
Notably, \our{} consistently outperforms strong baselines such as GenPRM-32B and Qwen2.5-Math-PRM-72B despite using only 8B parameters. 
On TB-DA, \our{} achieves the largest accruacy performance across each level of N, rising from 27.7\% at N${=}4$ to 34.3\% at N${=}32$. 
Moreover, while existing PRMs often suffer from performance bottlenecks beyond a certain response threshold (as observed in Section~\ref{sec:analyses}), \our{} continues to scale effectively, yielding consistent gains as the response group size increases.
For example, on TB-NR, Qwen2.5-Math-PRM-72B saturates after N${=}16$ (74.9\% $\rightarrow$ 75.3\%), whereas \our{} continues to improve from 74.2\% at N${=}8$ to 78.1\% at N${=}32$.
These results demonstrate that \our{} provides stronger reward supervision on LRMs' reasoning trajectories, therefore yielding better performance improvement compared with other step-verification baselines.

\begin{figure}[!h]
    \centering
    \includegraphics[width=0.85\linewidth]{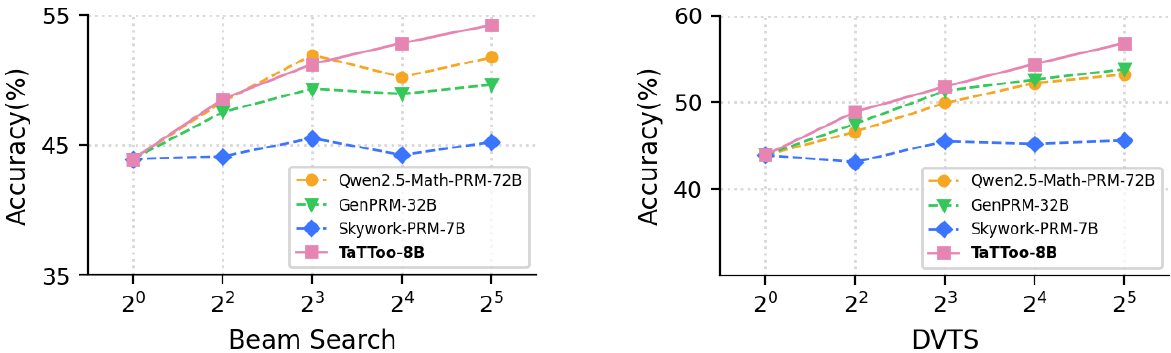}
    \caption{Performance of \our{} on two additional TTS strategies, Beam Search and Diverse Verifier Tree Search (DVTS). We report the average accuracy across all 5 tabular reasoning tasks.
    }
    \label{fig:main_res}
    \vspace{-5pt}
\end{figure}

\textbf{Generalizability on Other TTS Strategies.} 
Beyond best-of-$N$, we also evaluate \our{} under two additional TTS strategies (Beam Search and DVTS) and compare with the strongest PRM baselines. Figure~\ref{fig:main_res} reports the average performance across the five tabular reasoning tasks. Under each TTS strategy, \our{} consistently yields steady improvements as the number of responses N increases, whereas other baseline PRMs often plateau. For example, in beam search, \our{} improves from 45.0\% to 54.8\%, while GenPRM saturates around 51\% and Skywork-PRM remains below 46\%. These results highlight the strong generalizability of \our{} across diverse TTS strategies.

\subsection{In-depth Analyses on \our{}}

\textbf{Mastery of RL with Bootstrapping from SFT.}
To examine the respective roles of SFT and RL in \our{}'s dual-stage training paradigm, we compare against a variant \our{}~(SFT only), which is trained solely on the first SFT stage. 
As shown in Table~\ref{tab:ablation}, under the Best-of-N evaluation, the second-stage RL policy optimization consistently improves performance over the SFT-only initialization.
Specifically, we observe that the average accuracy across all three tasks improves from 72.3\% (SFT only) to 78.5\% after RL training, yielding a total gain of 10.2\%.
This demonstrates that bootstrapping from SFT provides a solid initialization, while RL optimization further enhances our PRM's reasoning and tool-use effectiveness during the verification process.

\begin{table}[!t]
    \centering
    \caption{In-depth analysis of \our{} on three table datasets. We evaluate the contributions of SFT and RL training stages, and assess the impact of reward shaping components during RL optimization.}
    \label{tab:ablation}
    \small
    \resizebox{\textwidth}{!}{
    \begin{tabular}{lcccccccccccc}
        \toprule
        \multirow{2}{*}{\textbf{Training Variants}} & \multicolumn{4}{c}{\textbf{TB-NR}} & \multicolumn{4}{c}{\textbf{TB-FC}} & \multicolumn{4}{c}{\textbf{TB-DA}}\\
        \cmidrule(lr){2-5} \cmidrule(lr){6-9} \cmidrule(lr){10-13}
        & 4 & 8 & 16 & 32 & 4 & 8 & 16 & 32 & 4 & 8 & 16 & 32\\
            \midrule[0.35pt]
            \textit{\textbf{\our{} (SFT only)}} & 
            67.9 & 69.1 & 72.0 & 73.7 &
            71.5 & 73.0 & 74.6 & 75.2 &
            23.3 & 25.6 & 26.2 & 26.4 \\
            \midrule[0.35pt]
            \rowcolor{gray!20}
            \textbf{\our{}} &
            \textbf{71.2} & \textbf{74.2} & \textbf{76.4} & \textbf{78.1} &
            \textbf{77.4} & \textbf{79.6} & \textbf{81.2} & \textbf{82.0} &
            \textbf{27.7} & \textbf{31.9} & \textbf{33.6} & \textbf{34.3} \\
            \textit{w/o tool-grounding} &
            68.5 & 71.1 & 72.7 & 74.6 &
            73.2 & 75.6 & 75.5 & 76.3 &
            26.2 & 28.1 & 28.7 & 30.3 \\
            
            \textit{w/o confidence caliboration} &
            71.1 & 73.7 & 74.3 & 76.2 &
            76.4 & 76.7 & 78.4 & 80.5 &
            27.4 & 29.5 & 31.3 & 33.2\\

            \textit{rule-based (GRPO)} &
            67.0 & 68.4 & 70.4 & 73.1 &
            71.6 & 74.0 & 74.9 & 75.8 &
            25.5 & 27.4 & 28.0 & 28.6\\
            
            \bottomrule
    \end{tabular}
    }
\end{table}

\begin{wrapfigure}{r}{0.41\textwidth}
  \centering
  \vspace{-15pt}
  \includegraphics[width=0.41\textwidth]{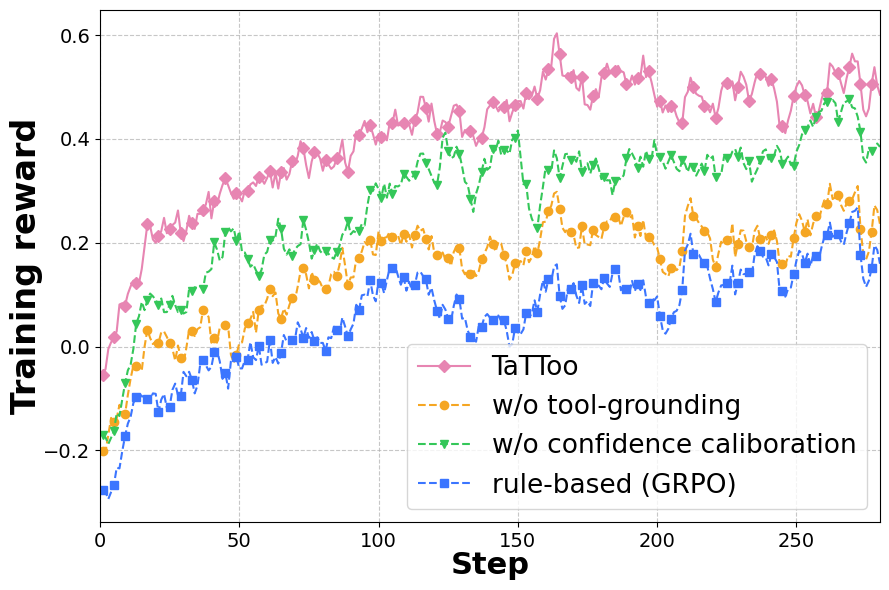}
  \caption{Training dynamics of \our{} and ablated variants. We report the training reward across 280 training steps.}
  \label{fig:training_dynamics}
  \vspace{-10pt}
\end{wrapfigure}

\textbf{Reward Shaping during RL Training.}  
Next, we analyze the effectiveness of each supervised component in our per-step reward signal $s_i$ design (Eq.~\ref{eqa:rewar_design}), with the ablation results reported in Table~\ref{tab:ablation}.
Removing the tool-grounding term yields the largest drop (e.g., $\downarrow $4.0\% on TB-DA at N${=}32$), highlighting its critical role in encouraging effective tool use during RL training. In addition, excluding confidence calibration reduces performance by 1.6\% on average, showing its complementary effect in stabilizing reward signals.
We also compare \our{} with the original rule-based group-relative reward from GRPO, which yields only marginal improvement over SFT. 
Finally, Figure~\ref{fig:training_dynamics} visualizes the training dynamics of \our{} and other variants during RL optimization.

\textbf{Additional Experiments.} Additional experiments on \our{} including ablations on the training coefficients and case studies are provided in Appendix~\ref{app:add_exps}.

\section{Related Works}
Reasoning over tables poses a unique challenge for LLMs, requiring them to bridge natural language understanding with structured reasoning over rows, columns, and cell values \citep{jin2022survey, zhang2025survey}.
Recent works \citep{tang2020rpt,iida2021tabbie,deng2022turl} have investigated tabular reasoning on several downstream tasks, including table QA \citep{wikitq, chen2020hybridqa}, table fact verification \citep{chen2019tabfact,parikh2020totto}, text-to-SQL \citep{mohammadjafari2024natural}, etc. 
Early-stage tabular reasoning methods, such as TAPAS \citep{herzig2020tapas} and TaBERT \citep{yin2020tabert}, encode table data into transformer-based encoder representations. Later methods leverage the capabilities of LLMs to apply either prompt engineering \citep{sui2023tap4llm, chain_of_table} or supervised fine-tuning techniques \citep{tablegpt2, zhang2023tablellama} for enhanced tabular reasoning. More recent works, including the Table-R1 series~\citep{wu2025table, yang2025tabler1inferencetimescalingtable, jin2025table} and Reasoning-Table~\citep{lei2025reasoning}, leverage RL to acquire higher-quality reasoning paths during reasoning over tables. 

While these recent advances have focused on improving the generation ability of models on tables, how to provide robust and verifiable reward supervision for the lengthy and complex output trajectories generated by the table-specific reasoning models remains largely unexplored \citep{zhang2024surveytablereasoninglarge}. 
This essential yet overlooked gap motivates us to develop the first tool-use and thinking PRM, which is specifically designed for enhancing test time scaling on tabular reasoning tasks. 
We leave additional related works on Process Reward Models and Tool Integration with RL in Appendix \ref{app:additional_related_work}.

\vspace{-7pt}
\section{Conclusion}
\vspace{-7pt}
We introduced \our{}, a novel tool-augmented thinking PRM tailored for tabular reasoning. By diagnosing why existing verifiers fail on table retrieval and schema interaction, we built a scalable pipeline with expert rationales, table prefixes, and tool-augmented verification, and trained our model via SFT followed by RL with reward shaping. \our{} achieves comparable performance across five table benchmarks, surpassing strong PRMs with up to 9× parameter efficiency and generalizing across multiple TTS strategies. Our results underscore the importance of table-grounded reward supervision and point toward future directions in reward modeling for structured reasoning tasks.

\bibliography{main}

\begin{thebibliography}{85}
\providecommand{\natexlab}[1]{#1}
\providecommand{\url}[1]{\texttt{#1}}
\expandafter\ifx\csname urlstyle\endcsname\relax
  \providecommand{\doi}[1]{doi: #1}\else
  \providecommand{\doi}{doi: \begingroup \urlstyle{rm}\Url}\fi

\bibitem[Akhtar et~al.(2023)Akhtar, Shankarampeta, Gupta, Patil, Cocarascu, and Simperl]{akhtar2023exploring}
Mubashara Akhtar, Abhilash Shankarampeta, Vivek Gupta, Arpit Patil, Oana Cocarascu, and Elena Simperl.
\newblock Exploring the numerical reasoning capabilities of language models: A comprehensive analysis on tabular data.
\newblock \emph{arXiv preprint arXiv:2311.02216}, 2023.

\bibitem[Anthropic(2025)]{anthropic2025_claude_opus_sonnet_syscard}
Anthropic.
\newblock Claude opus 4 and claude sonnet 4 system card.
\newblock Technical report, Anthropic, 2025.

\bibitem[Beeching et~al.(2024)Beeching, Tunstall, and Rush]{beeching2024scalingtesttimecompute}
Edward Beeching, Lewis Tunstall, and Sasha Rush.
\newblock Scaling test-time compute with open models, 2024.
\newblock URL \url{https://huggingface.co/spaces/HuggingFaceH4/blogpost-scaling-test-time-compute}.

\bibitem[Brown et~al.(2024)Brown, Juravsky, Ehrlich, Clark, Le, Ré, and Mirhoseini]{brown2024largelanguagemonkeysscaling}
Bradley Brown, Jordan Juravsky, Ryan Ehrlich, Ronald Clark, Quoc~V. Le, Christopher Ré, and Azalia Mirhoseini.
\newblock Large language monkeys: Scaling inference compute with repeated sampling, 2024.
\newblock URL \url{https://arxiv.org/abs/2407.21787}.

\bibitem[Chen et~al.(2019)Chen, Wang, Chen, Zhang, Wang, Li, Zhou, and Wang]{chen2019tabfact}
Wenhu Chen, Hongmin Wang, Jianshu Chen, Yunkai Zhang, Hong Wang, Shiyang Li, Xiyou Zhou, and William~Yang Wang.
\newblock Tabfact: A large-scale dataset for table-based fact verification.
\newblock \emph{arXiv preprint arXiv:1909.02164}, 2019.

\bibitem[Chen et~al.(2020)Chen, Zha, Chen, Xiong, Wang, and Wang]{chen2020hybridqa}
Wenhu Chen, Hanwen Zha, Zhiyu Chen, Wenhan Xiong, Hong Wang, and William Wang.
\newblock Hybridqa: A dataset of multi-hop question answering over tabular and textual data.
\newblock \emph{arXiv preprint arXiv:2004.07347}, 2020.

\bibitem[Chen et~al.(2025)Chen, Li, Wang, Jin, Qian, Wang, Wang, Zhang, Zhang, Zhang, Tong, and Ji]{chen2025rmr1rewardmodelingreasoning}
Xiusi Chen, Gaotang Li, Ziqi Wang, Bowen Jin, Cheng Qian, Yu~Wang, Hongru Wang, Yu~Zhang, Denghui Zhang, Tong Zhang, Hanghang Tong, and Heng Ji.
\newblock Rm-r1: Reward modeling as reasoning, 2025.
\newblock URL \url{https://arxiv.org/abs/2505.02387}.

\bibitem[Cui et~al.(2025)Cui, Yuan, Wang, Wang, Li, He, Fan, Yu, Xu, Chen, et~al.]{cui2025process}
Ganqu Cui, Lifan Yuan, Zefan Wang, Hanbin Wang, Wendi Li, Bingxiang He, Yuchen Fan, Tianyu Yu, Qixin Xu, Weize Chen, et~al.
\newblock Process reinforcement through implicit rewards.
\newblock \emph{arXiv preprint arXiv:2502.01456}, 2025.

\bibitem[Deng et~al.(2022)Deng, Sun, Lees, Wu, and Yu]{deng2022turl}
Xiang Deng, Huan Sun, Alyssa Lees, You Wu, and Cong Yu.
\newblock Turl: Table understanding through representation learning.
\newblock \emph{ACM SIGMOD Record}, 51\penalty0 (1):\penalty0 33--40, 2022.

\bibitem[Dong et~al.(2025)Dong, Mao, Ma, Bao, Chen, Wang, Chen, Du, Wang, Zhang, et~al.]{dong2025agentic}
Guanting Dong, Hangyu Mao, Kai Ma, Licheng Bao, Yifei Chen, Zhongyuan Wang, Zhongxia Chen, Jiazhen Du, Huiyang Wang, Fuzheng Zhang, et~al.
\newblock Agentic reinforced policy optimization.
\newblock \emph{arXiv preprint arXiv:2507.19849}, 2025.

\bibitem[Dong et~al.(2021)Dong, Cordonnier, and Loukas]{dong2021attention}
Yihe Dong, Jean-Baptiste Cordonnier, and Andreas Loukas.
\newblock Attention is not all you need: Pure attention loses rank doubly exponentially with depth.
\newblock In \emph{International conference on machine learning}, pages 2793--2803. PMLR, 2021.

\bibitem[Feng et~al.(2025{\natexlab{a}})Feng, Huang, Qu, Zhang, Qin, Zhong, Jiang, Chi, and Zhong]{feng2025retool}
Jiazhan Feng, Shijue Huang, Xingwei Qu, Ge~Zhang, Yujia Qin, Baoquan Zhong, Chengquan Jiang, Jinxin Chi, and Wanjun Zhong.
\newblock Retool: Reinforcement learning for strategic tool use in llms.
\newblock \emph{arXiv preprint arXiv:2504.11536}, 2025{\natexlab{a}}.

\bibitem[Feng et~al.(2025{\natexlab{b}})Feng, Huang, Qu, Zhang, Qin, Zhong, Jiang, Chi, and Zhong]{feng2025retoolreinforcementlearningstrategic}
Jiazhan Feng, Shijue Huang, Xingwei Qu, Ge~Zhang, Yujia Qin, Baoquan Zhong, Chengquan Jiang, Jinxin Chi, and Wanjun Zhong.
\newblock Retool: Reinforcement learning for strategic tool use in llms, 2025{\natexlab{b}}.
\newblock URL \url{https://arxiv.org/abs/2504.11536}.

\bibitem[Feng et~al.(2025{\natexlab{c}})Feng, Chen, Lu, Li, Cheng, Peng, Tang, Liu, and Zhang]{feng2025prm}
Zhangying Feng, Qianglong Chen, Ning Lu, Yongqian Li, Siqi Cheng, Shuangmu Peng, Duyu Tang, Shengcai Liu, and Zhirui Zhang.
\newblock Is prm necessary? problem-solving rl implicitly induces prm capability in llms.
\newblock \emph{arXiv preprint arXiv:2505.11227}, 2025{\natexlab{c}}.

\bibitem[Guan et~al.(2025)Guan, Zhang, Liu, Shang, Sun, Zhu, Yang, and Yang]{guan2025rstar}
Xinyu Guan, Li~Lyna Zhang, Yifei Liu, Ning Shang, Youran Sun, Yi~Zhu, Fan Yang, and Mao Yang.
\newblock rstar-math: Small llms can master math reasoning with self-evolved deep thinking.
\newblock \emph{arXiv preprint arXiv:2501.04519}, 2025.

\bibitem[Guo et~al.(2025)Guo, Yang, Zhang, Song, Zhang, Xu, Zhu, Ma, Wang, Bi, et~al.]{deepseek_r1}
Daya Guo, Dejian Yang, Haowei Zhang, Junxiao Song, Ruoyu Zhang, Runxin Xu, Qihao Zhu, Shirong Ma, Peiyi Wang, Xiao Bi, et~al.
\newblock Deepseek-r1: Incentivizing reasoning capability in llms via reinforcement learning.
\newblock \emph{arXiv preprint arXiv:2501.12948}, 2025.

\bibitem[He et~al.(2024{\natexlab{a}})He, Wei, Yan, Liu, Wang, Gan, Tu, Liu, Zeng, Wang, Wang, Li, Zhang, Xu, An, Liu, and Zhou]{skyworkopeno12024}
Jujie He, Tianwen Wei, Rui Yan, Jiacai Liu, Chaojie Wang, Yimeng Gan, Shiwen Tu, Chris~Yuhao Liu, Liang Zeng, Xiaokun Wang, Boyang Wang, Yongcong Li, Fuxiang Zhang, Jiacheng Xu, Bo~An, Yang Liu, and Yahui Zhou.
\newblock Skywork-o1 open series.
\newblock \url{https://huggingface.co/Skywork}, November 2024{\natexlab{a}}.
\newblock URL \url{https://huggingface.co/Skywork}.

\bibitem[He et~al.(2024{\natexlab{b}})He, Zou, Lin, Zhou, Han, Yuan, and Zhang]{he-etal-2024-cocost}
Xinyi He, Jiaru Zou, Yun Lin, Mengyu Zhou, Shi Han, Zejian Yuan, and Dongmei Zhang.
\newblock {C}o{C}o{ST}: Automatic complex code generation with online searching and correctness testing.
\newblock In Yaser Al-Onaizan, Mohit Bansal, and Yun-Nung Chen, editors, \emph{Proceedings of the 2024 Conference on Empirical Methods in Natural Language Processing}, pages 19433--19451, Miami, Florida, USA, November 2024{\natexlab{b}}. Association for Computational Linguistics.
\newblock \doi{10.18653/v1/2024.emnlp-main.1082}.
\newblock URL \url{https://aclanthology.org/2024.emnlp-main.1082/}.

\bibitem[Herzig et~al.(2020)Herzig, Nowak, M{\"u}ller, Piccinno, and Eisenschlos]{herzig2020tapas}
Jonathan Herzig, Pawe{\l}~Krzysztof Nowak, Thomas M{\"u}ller, Francesco Piccinno, and Julian~Martin Eisenschlos.
\newblock Tapas: Weakly supervised table parsing via pre-training.
\newblock \emph{arXiv preprint arXiv:2004.02349}, 2020.

\bibitem[Iida et~al.(2021)Iida, Thai, Manjunatha, and Iyyer]{iida2021tabbie}
Hiroshi Iida, Dung Thai, Varun Manjunatha, and Mohit Iyyer.
\newblock Tabbie: Pretrained representations of tabular data.
\newblock \emph{arXiv preprint arXiv:2105.02584}, 2021.

\bibitem[Jaech et~al.(2024)Jaech, Kalai, Lerer, Richardson, El-Kishky, Low, Helyar, Madry, Beutel, Carney, et~al.]{o1}
Aaron Jaech, Adam Kalai, Adam Lerer, Adam Richardson, Ahmed El-Kishky, Aiden Low, Alec Helyar, Aleksander Madry, Alex Beutel, Alex Carney, et~al.
\newblock Openai o1 system card.
\newblock \emph{arXiv preprint arXiv:2412.16720}, 2024.

\bibitem[Jin et~al.(2022)Jin, Siebert, Li, and Chen]{jin2022survey}
Nengzheng Jin, Joanna Siebert, Dongfang Li, and Qingcai Chen.
\newblock A survey on table question answering: recent advances.
\newblock In \emph{China Conference on Knowledge Graph and Semantic Computing}, pages 174--186. Springer, 2022.

\bibitem[Jin et~al.(2025)Jin, Xin, Xie, Li, Qi, Chen, Dai, Wu, and Haffari]{jin2025table}
Rihui Jin, Zheyu Xin, Xing Xie, Zuoyi Li, Guilin Qi, Yongrui Chen, Xinbang Dai, Tongtong Wu, and Gholamreza Haffari.
\newblock Table-r1: Self-supervised and reinforcement learning for program-based table reasoning in small language models.
\newblock \emph{arXiv preprint arXiv:2506.06137}, 2025.

\bibitem[Kakade and Langford(2002)]{kakade2002approximately}
Sham Kakade and John Langford.
\newblock Approximately optimal approximate reinforcement learning.
\newblock In \emph{Proceedings of the nineteenth international conference on machine learning}, pages 267--274, 2002.

\bibitem[Khalifa et~al.(2025)Khalifa, Agarwal, Logeswaran, Kim, Peng, Lee, Lee, and Wang]{khalifa2025processrewardmodelsthink}
Muhammad Khalifa, Rishabh Agarwal, Lajanugen Logeswaran, Jaekyeom Kim, Hao Peng, Moontae Lee, Honglak Lee, and Lu~Wang.
\newblock Process reward models that think, 2025.
\newblock URL \url{https://arxiv.org/abs/2504.16828}.

\bibitem[Lei et~al.(2025)Lei, Meng, Huang, Chen, Zhang, He, Zhao, and Liu]{lei2025reasoning}
Fangyu Lei, Jinxiang Meng, Yiming Huang, Tinghong Chen, Yun Zhang, Shizhu He, Jun Zhao, and Kang Liu.
\newblock Reasoning-table: Exploring reinforcement learning for table reasoning.
\newblock \emph{arXiv preprint arXiv:2506.01710}, 2025.

\bibitem[Li et~al.(2023{\natexlab{a}})Li, He, Yashar, Cui, Ge, Zhang, Fainman, Zhang, and Chaudhuri]{li2023table}
Peng Li, Yeye He, Dror Yashar, Weiwei Cui, Song Ge, Haidong Zhang, Danielle~Rifinski Fainman, Dongmei Zhang, and Surajit Chaudhuri.
\newblock Table-gpt: Table-tuned gpt for diverse table tasks.
\newblock \emph{arXiv preprint arXiv:2310.09263}, 2023{\natexlab{a}}.

\bibitem[Li and Li(2024)]{li2024process}
Wendi Li and Yixuan Li.
\newblock Process reward model with q-value rankings.
\newblock \emph{arXiv preprint arXiv:2410.11287}, 2024.

\bibitem[Li et~al.(2023{\natexlab{b}})Li, Xu, Chen, Huang, Li, Jiang, Li, Zhou, Zheng, and Shen]{DBLP:journals/corr/abs-2311-11268}
Yinghui Li, Zishan Xu, Shaoshen Chen, Haojing Huang, Yangning Li, Yong Jiang, Zhongli Li, Qingyu Zhou, Hai{-}Tao Zheng, and Ying Shen.
\newblock Towards real-world writing assistance: {A} chinese character checking benchmark with faked and misspelled characters.
\newblock \emph{CoRR}, abs/2311.11268, 2023{\natexlab{b}}.
\newblock \doi{10.48550/ARXIV.2311.11268}.
\newblock URL \url{https://doi.org/10.48550/arXiv.2311.11268}.

\bibitem[Li et~al.(2024)Li, Zhou, Luo, Ma, Li, Zheng, Hu, and Yu]{DBLP:journals/corr/abs-2402-11100}
Yinghui Li, Qingyu Zhou, Yuanzhen Luo, Shirong Ma, Yangning Li, Hai{-}Tao Zheng, Xuming Hu, and Philip~S. Yu.
\newblock When llms meet cunning questions: {A} fallacy understanding benchmark for large language models.
\newblock \emph{CoRR}, abs/2402.11100, 2024.
\newblock \doi{10.48550/ARXIV.2402.11100}.
\newblock URL \url{https://doi.org/10.48550/arXiv.2402.11100}.

\bibitem[Lightman et~al.(2024)Lightman, Kosaraju, Burda, Edwards, Baker, Lee, Leike, Schulman, Sutskever, and Cobbe]{lightman2024lets}
Hunter Lightman, Vineet Kosaraju, Yuri Burda, Harrison Edwards, Bowen Baker, Teddy Lee, Jan Leike, John Schulman, Ilya Sutskever, and Karl Cobbe.
\newblock Let's verify step by step.
\newblock In \emph{The Twelfth International Conference on Learning Representations}, 2024.
\newblock URL \url{https://openreview.net/forum?id=v8L0pN6EOi}.

\bibitem[Liu et~al.(2025)Liu, Gao, Zhao, Zhang, Li, Qi, Ouyang, and Zhou]{liu2025can}
Runze Liu, Junqi Gao, Jian Zhao, Kaiyan Zhang, Xiu Li, Biqing Qi, Wanli Ouyang, and Bowen Zhou.
\newblock Can 1b llm surpass 405b llm? rethinking compute-optimal test-time scaling.
\newblock \emph{arXiv preprint arXiv:2502.06703}, 2025.

\bibitem[Lu et~al.(2025)Lu, Chen, Liu, Thapa, Boen, and Zou]{lu2025octotoolsagenticframeworkextensible}
Pan Lu, Bowen Chen, Sheng Liu, Rahul Thapa, Joseph Boen, and James Zou.
\newblock Octotools: An agentic framework with extensible tools for complex reasoning, 2025.
\newblock URL \url{https://arxiv.org/abs/2502.11271}.

\bibitem[Maxwell-Jia(2024)]{maxwelljia2024aime24}
Maxwell-Jia.
\newblock {AIME 2024} dataset.
\newblock \url{https://huggingface.co/datasets/Maxwell-Jia/AIME_2024}, 2024.
\newblock Accessed: 2025-05-15.

\bibitem[Mohammadjafari et~al.(2024)Mohammadjafari, Maida, and Gottumukkala]{mohammadjafari2024natural}
Ali Mohammadjafari, Anthony~S Maida, and Raju Gottumukkala.
\newblock From natural language to sql: Review of llm-based text-to-sql systems.
\newblock \emph{arXiv preprint arXiv:2410.01066}, 2024.

\bibitem[Muennighoff et~al.(2025)Muennighoff, Yang, Shi, Li, Fei-Fei, Hajishirzi, Zettlemoyer, Liang, Candès, and Hashimoto]{muennighoff2025s1simpletesttimescaling}
Niklas Muennighoff, Zitong Yang, Weijia Shi, Xiang~Lisa Li, Li~Fei-Fei, Hannaneh Hajishirzi, Luke Zettlemoyer, Percy Liang, Emmanuel Candès, and Tatsunori Hashimoto.
\newblock s1: Simple test-time scaling, 2025.
\newblock URL \url{https://arxiv.org/abs/2501.19393}.

\bibitem[Paranjape et~al.(2023)Paranjape, Lundberg, Singh, Hajishirzi, Zettlemoyer, and Ribeiro]{paranjape2023art}
Bhargavi Paranjape, Scott Lundberg, Sameer Singh, Hannaneh Hajishirzi, Luke Zettlemoyer, and Marco~Tulio Ribeiro.
\newblock Art: Automatic multi-step reasoning and tool-use for large language models, 2023.
\newblock URL \url{https://arxiv.org/abs/2303.09014}.

\bibitem[Parikh et~al.(2020)Parikh, Wang, Gehrmann, Faruqui, Dhingra, Yang, and Das]{parikh2020totto}
Ankur Parikh, Xuezhi Wang, Sebastian Gehrmann, Manaal Faruqui, Bhuwan Dhingra, Diyi Yang, and Dipanjan Das.
\newblock Totto: A controlled table-to-text generation dataset.
\newblock In \emph{EMNLP 2020}, pages 1173--1186, 2020.

\bibitem[Pasupat and Liang(2015{\natexlab{a}})]{pasupat2015compositional}
Panupong Pasupat and Percy Liang.
\newblock Compositional semantic parsing on semi-structured tables.
\newblock In Chengqing Zong and Michael Strube, editors, \emph{Proceedings of the 53rd Annual Meeting of the Association for Computational Linguistics and the 7th International Joint Conference on Natural Language Processing (Volume 1: Long Papers)}, pages 1470--1480, Beijing, China, July 2015{\natexlab{a}}. Association for Computational Linguistics.
\newblock \doi{10.3115/v1/P15-1142}.
\newblock URL \url{https://aclanthology.org/P15-1142/}.

\bibitem[Pasupat and Liang(2015{\natexlab{b}})]{wikitq}
Panupong Pasupat and Percy Liang.
\newblock Compositional semantic parsing on semi-structured tables.
\newblock In \emph{Proceedings of the 53rd Annual Meeting of the Association for Computational Linguistics and the 7th International Joint Conference on Natural Language Processing of the Asian Federation of Natural Language Processing, {ACL} 2015, July 26-31, 2015, Beijing, China, Volume 1: Long Papers}, 2015{\natexlab{b}}.

\bibitem[Patil et~al.(2024)Patil, Zhang, Wang, and Gonzalez]{patil2024gorilla}
Shishir~G Patil, Tianjun Zhang, Xin Wang, and Joseph~E Gonzalez.
\newblock Gorilla: Large language model connected with massive apis.
\newblock \emph{Advances in Neural Information Processing Systems}, 37:\penalty0 126544--126565, 2024.

\bibitem[Qian et~al.(2025)Qian, Acikgoz, He, Wang, Chen, Hakkani-Tür, Tur, and Ji]{qian2025toolrlrewardtoollearning}
Cheng Qian, Emre~Can Acikgoz, Qi~He, Hongru Wang, Xiusi Chen, Dilek Hakkani-Tür, Gokhan Tur, and Heng Ji.
\newblock Toolrl: Reward is all tool learning needs, 2025.
\newblock URL \url{https://arxiv.org/abs/2504.13958}.

\bibitem[Qu et~al.(2025)Qu, Dai, Wei, Cai, Wang, Yin, Xu, and Wen]{Qu_2025}
Changle Qu, Sunhao Dai, Xiaochi Wei, Hengyi Cai, Shuaiqiang Wang, Dawei Yin, Jun Xu, and Ji-rong Wen.
\newblock Tool learning with large language models: a survey.
\newblock \emph{Frontiers of Computer Science}, 19\penalty0 (8), January 2025.
\newblock ISSN 2095-2236.
\newblock \doi{10.1007/s11704-024-40678-2}.
\newblock URL \url{http://dx.doi.org/10.1007/s11704-024-40678-2}.

\bibitem[Rein et~al.(2023)Rein, Hou, Stickland, Petty, Pang, Dirani, Michael, and Bowman]{rein2023gpqagraduatelevelgoogleproofqa}
David Rein, Betty~Li Hou, Asa~Cooper Stickland, Jackson Petty, Richard~Yuanzhe Pang, Julien Dirani, Julian Michael, and Samuel~R. Bowman.
\newblock Gpqa: A graduate-level google-proof q\&a benchmark, 2023.
\newblock URL \url{https://arxiv.org/abs/2311.12022}.

\bibitem[Schick et~al.(2023)Schick, Dwivedi-Yu, Dessi, Raileanu, Lomeli, Hambro, Zettlemoyer, Cancedda, and Scialom]{schick2023toolformer}
Timo Schick, Jane Dwivedi-Yu, Roberto Dessi, Roberta Raileanu, Maria Lomeli, Eric Hambro, Luke Zettlemoyer, Nicola Cancedda, and Thomas Scialom.
\newblock Toolformer: Language models can teach themselves to use tools.
\newblock In \emph{Thirty-seventh Conference on Neural Information Processing Systems}, 2023.
\newblock URL \url{https://openreview.net/forum?id=Yacmpz84TH}.

\bibitem[Schulman et~al.(2017)Schulman, Wolski, Dhariwal, Radford, and Klimov]{ppo}
John Schulman, Filip Wolski, Prafulla Dhariwal, Alec Radford, and Oleg Klimov.
\newblock Proximal policy optimization algorithms.
\newblock \emph{arXiv preprint arXiv:1707.06347}, 2017.

\bibitem[Seo et~al.(2025)Seo, Kwon, and Lee]{seo2025mt}
Kwangwook Seo, Donguk Kwon, and Dongha Lee.
\newblock Mt-raig: Novel benchmark and evaluation framework for retrieval-augmented insight generation over multiple tables.
\newblock \emph{arXiv preprint arXiv:2502.11735}, 2025.

\bibitem[Setlur et~al.(2024)Setlur, Nagpal, Fisch, Geng, Eisenstein, Agarwal, Agarwal, Berant, and Kumar]{setlur2024rewarding}
Amrith Setlur, Chirag Nagpal, Adam Fisch, Xinyang Geng, Jacob Eisenstein, Rishabh Agarwal, Alekh Agarwal, Jonathan Berant, and Aviral Kumar.
\newblock Rewarding progress: Scaling automated process verifiers for llm reasoning.
\newblock \emph{arXiv preprint arXiv:2410.08146}, 2024.

\bibitem[Shao et~al.(2024)Shao, Wang, Zhu, Xu, Song, Bi, Zhang, Zhang, Li, Wu, et~al.]{grpo}
Zhihong Shao, Peiyi Wang, Qihao Zhu, Runxin Xu, Junxiao Song, Xiao Bi, Haowei Zhang, Mingchuan Zhang, YK~Li, Y~Wu, et~al.
\newblock Deepseekmath: Pushing the limits of mathematical reasoning in open language models, 2024.
\newblock \emph{URL https://arxiv. org/abs/2402.03300}, 2024.

\bibitem[Shen(2024)]{shen2024llmtoolssurvey}
Zhuocheng Shen.
\newblock Llm with tools: A survey, 2024.
\newblock URL \url{https://arxiv.org/abs/2409.18807}.

\bibitem[Sheng et~al.(2024)Sheng, Zhang, Ye, Wu, Zhang, Zhang, Peng, Lin, and Wu]{sheng2024hybridflow}
Guangming Sheng, Chi Zhang, Zilingfeng Ye, Xibin Wu, Wang Zhang, Ru~Zhang, Yanghua Peng, Haibin Lin, and Chuan Wu.
\newblock Hybridflow: A flexible and efficient rlhf framework.
\newblock \emph{arXiv preprint arXiv: 2409.19256}, 2024.

\bibitem[Snell et~al.(2024)Snell, Lee, Xu, and Kumar]{snell2024scaling}
Charlie Snell, Jaehoon Lee, Kelvin Xu, and Aviral Kumar.
\newblock Scaling llm test-time compute optimally can be more effective than scaling model parameters.
\newblock \emph{arXiv preprint arXiv:2408.03314}, 2024.

\bibitem[Song et~al.(2023)Song, Xiong, Zhu, Wu, Qian, Song, Huang, Li, Wang, Yao, et~al.]{song2023restgpt}
Yifan Song, Weimin Xiong, Dawei Zhu, Wenhao Wu, Han Qian, Mingbo Song, Hailiang Huang, Cheng Li, Ke~Wang, Rong Yao, et~al.
\newblock Restgpt: Connecting large language models with real-world restful apis.
\newblock \emph{arXiv preprint arXiv:2306.06624}, 2023.

\bibitem[Su et~al.(2024)Su, Wang, Ye, Zhou, Zhang, Zhu, Wang, Xu, Chen, Li, et~al.]{tablegpt2}
Aofeng Su, Aowen Wang, Chao Ye, Chen Zhou, Ga~Zhang, Guangcheng Zhu, Haobo Wang, Haokai Xu, Hao Chen, Haoze Li, et~al.
\newblock Tablegpt2: A large multimodal model with tabular data integration.
\newblock \emph{arXiv preprint arXiv:2411.02059}, 2024.

\bibitem[Sui et~al.(2023)Sui, Zou, Zhou, He, Du, Han, and Zhang]{sui2023tap4llm}
Yuan Sui, Jiaru Zou, Mengyu Zhou, Xinyi He, Lun Du, Shi Han, and Dongmei Zhang.
\newblock Tap4llm: Table provider on sampling, augmenting, and packing semi-structured data for large language model reasoning.
\newblock \emph{arXiv preprint arXiv:2312.09039}, 2023.

\bibitem[Sui et~al.(2024)Sui, Zhou, Zhou, Han, and Zhang]{table_meet_llm}
Yuan Sui, Mengyu Zhou, Mingjie Zhou, Shi Han, and Dongmei Zhang.
\newblock Table meets llm: Can large language models understand structured table data? a benchmark and empirical study.
\newblock In \emph{Proceedings of the 17th ACM International Conference on Web Search and Data Mining}, pages 645--654, 2024.

\bibitem[Tang et~al.(2020)Tang, Fan, Li, Tu, Du, Li, Madden, and Ouzzani]{tang2020rpt}
Nan Tang, Ju~Fan, Fangyi Li, Jianhong Tu, Xiaoyong Du, Guoliang Li, Sam Madden, and Mourad Ouzzani.
\newblock Rpt: relational pre-trained transformer is almost all you need towards democratizing data preparation.
\newblock \emph{arXiv preprint arXiv:2012.02469}, 2020.

\bibitem[Uesato et~al.(2022)Uesato, Kushman, Kumar, Song, Siegel, Wang, Creswell, Irving, and Higgins]{uesato2022solvingmathwordproblems}
Jonathan Uesato, Nate Kushman, Ramana Kumar, Francis Song, Noah Siegel, Lisa Wang, Antonia Creswell, Geoffrey Irving, and Irina Higgins.
\newblock Solving math word problems with process- and outcome-based feedback, 2022.
\newblock URL \url{https://arxiv.org/abs/2211.14275}.

\bibitem[Vakulenko and Savenkov(2017)]{vakulenko2017tableqa}
Svitlana Vakulenko and Vadim Savenkov.
\newblock Tableqa: Question answering on tabular data.
\newblock \emph{arXiv preprint arXiv:1705.06504}, 2017.

\bibitem[Wang et~al.(2024{\natexlab{a}})Wang, Fang, Wan, Wen, Zhu, Liu, Gong, Song, Chen, Ni, et~al.]{wang2024openr}
Jun Wang, Meng Fang, Ziyu Wan, Muning Wen, Jiachen Zhu, Anjie Liu, Ziqin Gong, Yan Song, Lei Chen, Lionel~M Ni, et~al.
\newblock Openr: An open source framework for advanced reasoning with large language models.
\newblock \emph{arXiv preprint arXiv:2410.09671}, 2024{\natexlab{a}}.

\bibitem[Wang et~al.(2024{\natexlab{b}})Wang, Li, Shao, Xu, Dai, Li, Chen, Wu, and Sui]{wang2024mathshepherdverifyreinforcellms}
Peiyi Wang, Lei Li, Zhihong Shao, R.~X. Xu, Damai Dai, Yifei Li, Deli Chen, Y.~Wu, and Zhifang Sui.
\newblock Math-shepherd: Verify and reinforce llms step-by-step without human annotations, 2024{\natexlab{b}}.
\newblock URL \url{https://arxiv.org/abs/2312.08935}.

\bibitem[Wang et~al.(2024{\natexlab{c}})Wang, Zhang, Li, Eisenschlos, Perot, Wang, Miculicich, Fujii, Shang, Lee, et~al.]{chain_of_table}
Zilong Wang, Hao Zhang, Chun-Liang Li, Julian~Martin Eisenschlos, Vincent Perot, Zifeng Wang, Lesly Miculicich, Yasuhisa Fujii, Jingbo Shang, Chen-Yu Lee, et~al.
\newblock Chain-of-table: Evolving tables in the reasoning chain for table understanding.
\newblock \emph{arXiv preprint arXiv:2401.04398}, 2024{\natexlab{c}}.

\bibitem[Wu et~al.(2025{\natexlab{a}})Wu, Yang, Li, Ji, Okumura, and Zhang]{wummqa}
Jian Wu, Linyi Yang, Dongyuan Li, Yuliang Ji, Manabu Okumura, and Yue Zhang.
\newblock Mmqa: Evaluating llms with multi-table multi-hop complex questions.
\newblock In \emph{The Thirteenth International Conference on Learning Representations}, 2025{\natexlab{a}}.

\bibitem[Wu et~al.(2024)Wu, Yang, Chai, Zhang, Liu, Du, Liang, Shu, Cheng, Sun, et~al.]{tablebench}
Xianjie Wu, Jian Yang, Linzheng Chai, Ge~Zhang, Jiaheng Liu, Xinrun Du, Di~Liang, Daixin Shu, Xianfu Cheng, Tianzhen Sun, et~al.
\newblock Tablebench: A comprehensive and complex benchmark for table question answering.
\newblock \emph{arXiv preprint arXiv:2408.09174}, 2024.

\bibitem[Wu et~al.(2025{\natexlab{b}})Wu, Yang, Liu, Wu, Pan, Zhang, Zhao, Song, Li, and Li]{wu2025table}
Zhenhe Wu, Jian Yang, Jiaheng Liu, Xianjie Wu, Changzai Pan, Jie Zhang, Yu~Zhao, Shuangyong Song, Yongxiang Li, and Zhoujun Li.
\newblock Table-r1: Region-based reinforcement learning for table understanding.
\newblock \emph{arXiv preprint arXiv:2505.12415}, 2025{\natexlab{b}}.

\bibitem[Yang et~al.(2024)Yang, Zhang, Hui, Gao, Yu, Li, Liu, Tu, Zhou, Lin, et~al.]{qwen2.5mathprm72b}
An~Yang, Beichen Zhang, Binyuan Hui, Bofei Gao, Bowen Yu, Chengpeng Li, Dayiheng Liu, Jianhong Tu, Jingren Zhou, Junyang Lin, et~al.
\newblock Qwen2.5-math technical report: Toward mathematical expert model via self-improvement.
\newblock \emph{arXiv preprint arXiv:2409.12122}, 2024.

\bibitem[Yang et~al.(2025{\natexlab{a}})Yang, Li, Yang, Zhang, Hui, Zheng, Yu, Gao, Huang, Lv, et~al.]{qwen3}
An~Yang, Anfeng Li, Baosong Yang, Beichen Zhang, Binyuan Hui, Bo~Zheng, Bowen Yu, Chang Gao, Chengen Huang, Chenxu Lv, et~al.
\newblock Qwen3 technical report, 2025{\natexlab{a}}.
\newblock URL \url{https://arxiv.org/abs/2505.09388}.

\bibitem[Yang et~al.(2020)Yang, Ma, Zhang, Li, and Zhou]{soft_template}
Jian Yang, Shuming Ma, Dongdong Zhang, Zhoujun Li, and Ming Zhou.
\newblock Improving neural machine translation with soft template prediction.
\newblock In Dan Jurafsky, Joyce Chai, Natalie Schluter, and Joel~R. Tetreault, editors, \emph{Proceedings of the 58th Annual Meeting of the Association for Computational Linguistics, {ACL} 2020, Online, July 5-10, 2020}, pages 5979--5989. Association for Computational Linguistics, 2020.
\newblock \doi{10.18653/V1/2020.ACL-MAIN.531}.
\newblock URL \url{https://doi.org/10.18653/v1/2020.acl-main.531}.

\bibitem[Yang et~al.(2025{\natexlab{b}})Yang, Chen, Cohan, and Zhao]{yang2025tabler1inferencetimescalingtable}
Zheyuan Yang, Lyuhao Chen, Arman Cohan, and Yilun Zhao.
\newblock Table-r1: Inference-time scaling for table reasoning, 2025{\natexlab{b}}.
\newblock URL \url{https://arxiv.org/abs/2505.23621}.

\bibitem[Yao et~al.(2023)Yao, Zhao, Yu, Du, Shafran, Narasimhan, and Cao]{yao2022react}
Shunyu Yao, Jeffrey Zhao, Dian Yu, Nan Du, Izhak Shafran, Karthik Narasimhan, and Yuan Cao.
\newblock React: Synergizing reasoning and acting in language models, 2023.
\newblock URL \url{https://arxiv.org/abs/2210.03629}.

\bibitem[Ye et~al.(2025)Ye, Huang, Xiao, Chern, Xia, and Liu]{ye2025limo}
Yixin Ye, Zhen Huang, Yang Xiao, Ethan Chern, Shijie Xia, and Pengfei Liu.
\newblock Limo: Less is more for reasoning.
\newblock \emph{arXiv preprint arXiv:2502.03387}, 2025.

\bibitem[Yin et~al.(2020)Yin, Neubig, Yih, and Riedel]{yin2020tabert}
Pengcheng Yin, Graham Neubig, Wen-tau Yih, and Sebastian Riedel.
\newblock Tabert: Pretraining for joint understanding of textual and tabular data.
\newblock \emph{arXiv preprint arXiv:2005.08314}, 2020.

\bibitem[Yu et~al.(2018)Yu, Zhang, Yang, Yasunaga, Wang, Li, Ma, Li, Yao, Roman, et~al.]{yu2018spider}
Tao Yu, Rui Zhang, Kai Yang, Michihiro Yasunaga, Dongxu Wang, Zifan Li, James Ma, Irene Li, Qingning Yao, Shanelle Roman, et~al.
\newblock Spider: A large-scale human-labeled dataset for complex and cross-domain semantic parsing and text-to-sql task.
\newblock In \emph{EMNLP 2018}, pages 3911--3921, 2018.

\bibitem[Zhang et~al.(2024{\natexlab{a}})Zhang, Zhoubian, Hu, Yue, Dong, and Tang]{zhang2024rest}
Dan Zhang, Sining Zhoubian, Ziniu Hu, Yisong Yue, Yuxiao Dong, and Jie Tang.
\newblock Rest-mcts*: Llm self-training via process reward guided tree search.
\newblock \emph{arXiv preprint arXiv:2406.03816}, 2024{\natexlab{a}}.

\bibitem[Zhang et~al.(2023)Zhang, Yue, Li, and Sun]{zhang2023tablellama}
Tianshu Zhang, Xiang Yue, Yifei Li, and Huan Sun.
\newblock Tablellama: Towards open large generalist models for tables.
\newblock \emph{arXiv preprint arXiv:2311.09206}, 2023.

\bibitem[Zhang et~al.(2024{\natexlab{b}})Zhang, Wang, Dou, Zhu, and Che]{zhang2024surveytablereasoninglarge}
Xuanliang Zhang, Dingzirui Wang, Longxu Dou, Qingfu Zhu, and Wanxiang Che.
\newblock A survey of table reasoning with large language models, 2024{\natexlab{b}}.
\newblock URL \url{https://arxiv.org/abs/2402.08259}.

\bibitem[Zhang et~al.(2025{\natexlab{a}})Zhang, Wang, Dou, Zhu, and Che]{zhang2025survey}
Xuanliang Zhang, Dingzirui Wang, Longxu Dou, Qingfu Zhu, and Wanxiang Che.
\newblock A survey of table reasoning with large language models.
\newblock \emph{Frontiers of Computer Science}, 19\penalty0 (9):\penalty0 199348, 2025{\natexlab{a}}.

\bibitem[Zhang et~al.(2025{\natexlab{b}})Zhang, Zheng, Wu, Zhang, Lin, Yu, Liu, Zhou, and Lin]{zhang2025lessons}
Zhenru Zhang, Chujie Zheng, Yangzhen Wu, Beichen Zhang, Runji Lin, Bowen Yu, Dayiheng Liu, Jingren Zhou, and Junyang Lin.
\newblock The lessons of developing process reward models in mathematical reasoning.
\newblock \emph{arXiv preprint arXiv:2501.07301}, 2025{\natexlab{b}}.

\bibitem[Zhao et~al.(2025)Zhao, Liu, Zhang, Zhou, Gao, Li, Lyu, Qian, Qi, Li, and Zhou]{zhao2025genprmscalingtesttimecompute}
Jian Zhao, Runze Liu, Kaiyan Zhang, Zhimu Zhou, Junqi Gao, Dong Li, Jiafei Lyu, Zhouyi Qian, Biqing Qi, Xiu Li, and Bowen Zhou.
\newblock Genprm: Scaling test-time compute of process reward models via generative reasoning, 2025.
\newblock URL \url{https://arxiv.org/abs/2504.00891}.

\bibitem[Zheng et~al.(2023)Zheng, Chiang, Sheng, Zhuang, Wu, Zhuang, Lin, Li, Li, Xing, et~al.]{zheng2023judging}
Lianmin Zheng, Wei-Lin Chiang, Ying Sheng, Siyuan Zhuang, Zhanghao Wu, Yonghao Zhuang, Zi~Lin, Zhuohan Li, Dacheng Li, Eric Xing, et~al.
\newblock Judging llm-as-a-judge with mt-bench and chatbot arena.
\newblock \emph{Advances in Neural Information Processing Systems}, 36:\penalty0 46595--46623, 2023.

\bibitem[Zheng et~al.(2024)Zheng, Zhang, Zhang, Ye, Luo, Feng, and Ma]{zheng2024llamafactory}
Yaowei Zheng, Richong Zhang, Junhao Zhang, Yanhan Ye, Zheyan Luo, Zhangchi Feng, and Yongqiang Ma.
\newblock Llamafactory: Unified efficient fine-tuning of 100+ language models.
\newblock In \emph{Proceedings of the 62nd Annual Meeting of the Association for Computational Linguistics (Volume 3: System Demonstrations)}, Bangkok, Thailand, 2024. Association for Computational Linguistics.
\newblock URL \url{http://arxiv.org/abs/2403.13372}.

\bibitem[Zhong et~al.(2025)Zhong, Shen, Li, Gao, Lu, Chen, Zhang, Zhou, Gu, and Zou]{zhong2025comprehensive}
Jialun Zhong, Wei Shen, Yanzeng Li, Songyang Gao, Hua Lu, Yicheng Chen, Yang Zhang, Wei Zhou, Jinjie Gu, and Lei Zou.
\newblock A comprehensive survey of reward models: Taxonomy, applications, challenges, and future.
\newblock \emph{arXiv preprint arXiv:2504.12328}, 2025.

\bibitem[Zhong et~al.(2017)Zhong, Xiong, and Socher]{zhong2017seq2sql}
Victor Zhong, Caiming Xiong, and Richard Socher.
\newblock Seq2sql: Generating structured queries from natural language using reinforcement learning.
\newblock \emph{CoRR}, 2017.

\bibitem[Zou et~al.(2025{\natexlab{a}})Zou, Fu, Chen, He, Li, Zhu, Han, and He]{zou2025gtrgraphtableragcrosstablequestion}
Jiaru Zou, Dongqi Fu, Sirui Chen, Xinrui He, Zihao Li, Yada Zhu, Jiawei Han, and Jingrui He.
\newblock Gtr: Graph-table-rag for cross-table question answering, 2025{\natexlab{a}}.
\newblock URL \url{https://arxiv.org/abs/2504.01346}.

\bibitem[Zou et~al.(2025{\natexlab{b}})Zou, Yang, Gu, Qiu, Shen, He, and Wang]{zou2025reasonflux}
Jiaru Zou, Ling Yang, Jingwen Gu, Jiahao Qiu, Ke~Shen, Jingrui He, and Mengdi Wang.
\newblock Reasonflux-prm: Trajectory-aware prms for long chain-of-thought reasoning in llms.
\newblock \emph{arXiv preprint arXiv:2506.18896}, 2025{\natexlab{b}}.

\end{thebibliography}
\bibliographystyle{plainnat}

\clearpage

\appendix
\renewcommand{\contentsname}{\Large Table of Contents}
{\hypersetup{linkcolor=black}
\tableofcontents
}

\addtocontents{toc}{\protect\setcounter{tocdepth}{2}}

\newpage 
\textbf{\Large Appendix}

\section{Additional Related Work}
\label{app:additional_related_work}

\paragraph{Table Question Answering.} The evolution of Table Question Answering (Table QA) research~\citep{jin2022survey} has been propelled by the creation of sophisticated evaluation resources that facilitate semantic parsing capabilities~\citep{soft_template,DBLP:journals/corr/abs-2311-11268, DBLP:journals/corr/abs-2402-11100}.
Foundational works, including WTQ~\citep{pasupat2015compositional} and TabFact~\citep{chen2019tabfact}, established initial evaluation paradigms through Wikipedia-derived HTML table QA pairs. 
Structured supervision has also been explored in alternative benchmarks such as WikiSQL~\citep{zhong2017seq2sql} and Spider~\citep{yu2018spider}, where logical expressions serve as explicit annotations to encourage systematic reasoning. More recent studies, such as MultiTableQA \citep{tablebench,zou2025gtrgraphtableragcrosstablequestion}, MT-RAIG \citep{seo2025mt}, and MMQA \citep{wummqa} has shifted towards multi-hop reasoning.

\paragraph{PRMs for Test-time Scaling.} 
Process Reward Models (PRMs) \citep{lightman2024lets, uesato2022solvingmathwordproblems, zhang2024rest} deliver fine-grained, step-level feedback to guide model reasoning, assigning intermediate rewards to individual reasoning steps rather than only judging final answers \citep{guan2025rstar, chen2025rmr1rewardmodelingreasoning}. Prominent PRMs, including Math-Shepherd \citep{wang2024mathshepherdverifyreinforcellms}, Skywork-PRM \citep{skyworkopeno12024}, and the Qwen2.5-Math-PRM family \citep{zhang2025lessons}, are trained using a mix of human annotations and synthesized supervision to score model-generated solution steps across domains such as math \citep{maxwelljia2024aime24}, scientific reasoning \citep{rein2023gpqagraduatelevelgoogleproofqa}, and programming \citep{he-etal-2024-cocost}; more recently, Think-PRM proposes a generative verifier to produce long-chain CoT evaluations \citep{khalifa2025processrewardmodelsthink}. PRMs have been incorporated into training-time optimization as reward signals via step-verified online RL and verifier-guided self-training \citep{li2024process, guan2025rstar, cui2025process}, and into inference-time scaling by coupling step-level scoring with search/decoding strategies \citep{zhao2025genprmscalingtesttimecompute, khalifa2025processrewardmodelsthink}, including beam search, reward-guided tree search, and Best-of-N sampling. 

\paragraph{Discriminative vs. Generative PRM.} In general, PRMs can be categorized as discriminative and generative evaluators \citep{zhong2025comprehensive}. A \textbf{discriminative PRM} treats verification as classification, directly predicting the correctness of each reasoning step with a scalar score. It is typically trained on step-level labels using cross-entropy loss, making it heavily reliant on step-level reward annotations. A \textbf{generative PRM} instead frames verification as conditional generation. It is trained with the standard language modeling objective to first generate rationales and then verify each step’s correctness via a judgment token (e.g., [correct, incorrect]).

\paragraph{Tool Integration with LLMs.} 
Recent research has explored augmenting LLMs with external tools to extend their reasoning and problem-solving abilities \citep{shen2024llmtoolssurvey, Qu_2025}. Early approaches rely on predefined APIs or plugins \citep{schick2023toolformer, paranjape2023art, yao2022react}, enabling models to call external functions during inference. Additional methods \citep{song2023restgpt, patil2024gorilla, lu2025octotoolsagenticframeworkextensible} emphasize training LLMs to select, invoke, and compose tools dynamically. More recent frameworks, such as ReTool \citep{feng2025retoolreinforcementlearningstrategic} and ToolRL \citep{qian2025toolrlrewardtoollearning}, extend tool learning with RL-based policy optimization \citep{ppo,grpo,dong2025agentic}, enabling reward-driven tool integration where LLMs iteratively refine their tool-use strategies. While prior work integrates tools into LLMs to improve model generation capabilities, integrating tools in the reward modeling process for reliable verification has remained underexplored, with existing PRMs still functioning primarily as text-only verifiers. Our work bridges this gap by proposing a tool-augmented PRM that leverages external tool executions for more reliable and precise reward supervision, enabling stronger verification of reasoning trajectories.

\clearpage
\section{Detailed Error Analysis}
\label{app:error_analysis}

\begin{figure}[!h]
    \centering
    \includegraphics[width=0.8\linewidth]{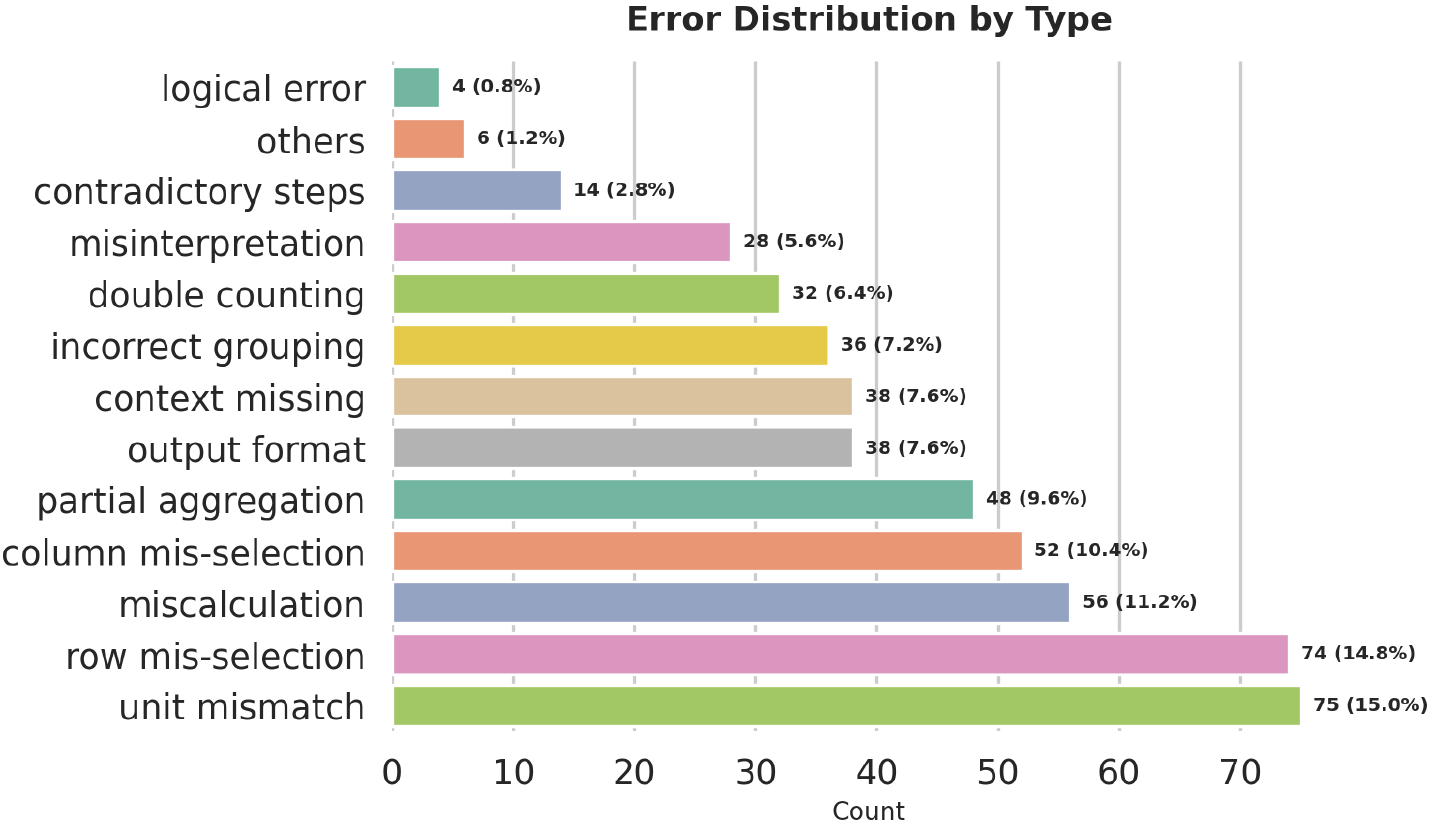}
    \caption{Error distribution over 500 incorrect LRM responses after Best-of-N. The errors are grouped into 13 predefined types, with the majority arising from table retrieval and schema interaction.}
    \label{fig:error_type}
\end{figure}

In Section \ref{sec:analyses}, we perform a fine-grained error analysis on 500 erroneous responses sampled after Best-of-$N$ selection with Qwen2.5-Math-PRM-72B, to better understand the limitations of LRMs and PRMs. Each response is inspected and categorized by human experts into 13 predefined error types, covering both reasoning and table-specific mistakes. Figure~\ref{fig:error_type} illustrates the overall error distribution.

\paragraph{Error Type Distribution.}
The most frequent errors are \textit{unit mismatch} (15.0\%), \textit{row mis-selection} (14.8\%), and \textit{miscalculation} (11.2\%). Other common issues include \textit{column mis-selection} (10.4\%), \textit{partial aggregation} (9.6\%), and missing or incomplete \textit{context} (7.6\%). Less frequent but still notable categories include \textit{output format errors}, \textit{incorrect grouping}, \textit{double counting}, \textit{misinterpretation}, and \textit{contradictory steps}. A small portion of errors is grouped under \textit{others} and \textit{logical errors}. This diverse distribution highlights that model failures are not restricted to arithmetic slips but extend to schema understanding and structural reasoning.

\paragraph{Mapping to Reasoning-Step Categories.}
To reveal deeper patterns, we align the 13 error types with four reasoning-step categories reflecting the typical flow of LRMs:
\begin{itemize}[leftmargin=*]
\item \textbf{Table Retrieval Step}: Includes row/column mis-selection, unit mismatch, and partial aggregation. These account for 47.7\% of total errors, indicating difficulty in locating and extracting the correct table region.
\item \textbf{Schema Interaction Step}: Covers miscalculation, grouping mistakes, double counting, and misinterpretation of table semantics. This represents 34.3\% of errors, reflecting challenges in reasoning over structured contents once retrieved.
\item \textbf{Inner-Thinking Step}: Logical errors or contradictory reasoning steps independent of table contents. These contribute 12.0\% of total errors, suggesting LRMs remain relatively competent in pure logical chains compared to table-centric operations.
\item \textbf{Others}: Errors arising from context omission or improper output formatting.
\end{itemize}

\paragraph{Key Findings.}
The analysis reveals that most model weaknesses lie in table-related operations, including table retrieval and schema interaction, rather than general logical reasoning. PRMs, when supervising such steps, face greater challenges since they must not only validate the correctness of reasoning but also verify alignment between the retrieved sub-table and the query.

\newpage

\section{\our{} Data Curation Pipline}
\label{app:data_curation}

We design a large-scale data curation pipeline that simulates real-world scenarios of PRM tool use and step verification at scale. As illustrated in Figure \ref{fig:data_pipline}, there are three main stages:

\textbf{Reasoning Trajectory Generation.} 
We begin by collecting trajectory responses from expert LRMs (e.g., DeepSeek-R1 \citep{deepseek_r1} and Claude-Opus-4.1 \citep{anthropic2025_claude_opus_sonnet_syscard}) on table-based questions drawn from diverse benchmarks, including TableInstruct~\citep{tablebench}, HybridQA~\citep{chen2020hybridqa}, ToTTo~\citep{parikh2020totto}, and WikiTQ~\citep{wikitq}.

We generate multiple model responses per query, capturing both correct and incorrect reasoning patterns. We then adopt a dual-verification procedure \citep{feng2025retool}, where both human annotators and expert LLMs are employed to examine and filter out low-quality or incomplete CoT data. Through this, we receive a high-quality set of LRMs' output responses $\mathcal{T_{\text{pool}}}$ for subsequent data labeling.

\textbf{Verification Synthesis \& Reward Assignment.} Our next step is to provide step-level verification rationales and assign PRM step-reward labels for each candidate response in $\mathcal{T}_{\text{pool}}$. To this end, we first identify the table retrieval and schema interaction steps within each response in $\mathcal{T}_{\text{pool}}$:
\begin{itemize}[leftmargin=*]

    \item 
\textit{Table retrieval steps} - We first extract the retrieved sub-table from each step. Then we apply LLM-as-a-judge to evaluate whether retrieved contents are accurate and provide complete rationales for the judgment. We assign step-level table reward $r_{i,\text{tab}} \in \{-1,1\}$ (in Eq.~\ref{eqa:prm_reward}) as the meaning of $\{\text{incorrect}, \text{correct}\}$ based on the correctness of the retrieval. This reward supervision explicitly trains PRMs to recognize if the retrieved sub-table aligns with the input query, addressing the limitation in \textit{Takeaway 1}.

    \item
\textit{Schema interaction steps} - We collect the sub-table retrieved from the preceding table retrieval step and use it as a table prefix. If the retrieval is incorrect, we manually replace it with the correct sub-table corresponding to the query. We then prepend this table prefix to the verification rationale generated by LLM-as-a-judge. Finally, we assign the PRM’s step-level table reward $r_{i,\text{tab}} \in \{-1,1\}$ as the meaning of $\{\text{incorrect}, \text{correct}\}$ based on the correctness of the schema interaction. By explicitly attaching the retrieved sub-table to each schema interaction step, we mitigate the dependencies issue noted in \textit{Takeaway~2}.

    \item 
\textit{Other steps without table operations involved} - We directly query an expert LLM (DeepSeek-R1) to generate verification rationales. We assign the PRM’s step-level reasoning reward $r_{i,\text{rea}} \in \{-1,1\}$ as the meaning of $\{\text{incorrect}, \text{correct}\}$ based on the correctness of the reasoning.
\end{itemize}

\textbf{Tool Use Synthesis.} 
To help PRMs learn to leverage tools for more accurate verification, we augment the collected verification rationales by incorporating tool invocation, execution, and feedback into the verification steps. Specifically, whenever the model’s inner reasoning involves a calculation or table lookup operation, we replace it with the corresponding tool call and its execution result. We primarily employ two types of tools:

\begin{itemize}[leftmargin=*]

    \item 
\textit{Computation tools -} Applying Python or SQL code snippets for arithmetic or aggregation operations. E.g., if a step verifies the sum of a table column, we replace the model’s manual calculation with a code snippet that executes the summation and returns the result.

    \item 
\textit{Table lookup tools -} Locating and extracting specific rows, columns, or cells from the table. E.g., if a step requires referencing a sub-table cell value during the verification, we replace the model’s self-extraction with an explicit lookup tool call that retrieves the corresponding entry.
\end{itemize}

By integrating verification processes with code snippets and real-time interpreter feedback, we construct roughly 60k data for \our{}'s verification reasoning and tool usage.

\newpage

\newpage

\section{Proof of Theorem \ref{theorem:improvement}}
\label{app:proof}

\paragraph{Notational conventions.}
We use $\bs_i$ for a state, $a_i$ for an action, $\pi$ for the current policy, and $\pi'$ for the updated policy.
The advantage is defined as 
\begin{equation}
   A^\pi(\bs_i,a_i)=Q^\pi(\bs_i,a_i)-V^\pi(\bs_i).
\end{equation}
The PRM signal at a step is the overall process reward, defined in Eq.\ref{eqa:prm_reward_}.
For a fixed $\bs_i$, we write
$\E_\pi[\cdot] \equiv \E_{a_i\sim \pi(\cdot\mid \bs_i)}[\cdot]$,
$\Var_\pi[\cdot] \equiv \Var_{a_i\sim \pi(\cdot\mid \bs_i)}[\cdot]$,
and $\Cov_\pi(r_{i,\text{rea}}(\bs_i,a_i),r_{i,\text{rea}}(\bs_i,a_i)) \equiv \Cov_{a_i\sim \pi(\cdot\mid \bs_i)}(r_{i,\text{rea}}(\bs_i,a_i),r_{i,\text{rea}}(\bs_i,a_i))$
Expectations over states use the subscript explicitly, e.g., $\E_{\bs_i\sim \rho}[\cdot]$.
We use $d_\rho^{\pi'}$ for the state distribution induced by the policy $\pi'$ starting from the initial distribution $\rho$.
Finally, $X \gtrsim Y$ means there exists a universal constant $c>0$, independent of $(\pi,\pi',\bs_i)$, such that $X \ge c\,Y$.

We start the proof by introducing two standard lemmas that will be used repeatedly; both are well-known results in the RL literature, and we omit their proofs here for brevity.

\begin{lemma}[\textbf{Performance Difference Lemma (PDL)}]
\label{lem:pdl}
For any pair of policies $\pi$ and $\pi'$ defined over the same Markov decision process with initial state distribution $\rho$, the following identity holds:
\begin{align*}
  \E_{\bs_i\sim \rho}\!\left[V^{\pi'}(\bs_i)-V^{\pi}(\bs_i)\right]
  \;=\;
  \E_{\bs_i\sim d_{\rho}^{\pi'}} \E_{a_i\sim \pi'(\cdot\mid \bs_i)}\!\big[A^{\pi}(\bs_i,a_i)\big].
\end{align*}
\end{lemma}

See proof of Lemma 6.1 in \citep{kakade2002approximately}.

\begin{lemma}[\textbf{Natural policy gradient (NPG) update form}]
\label{lem:npg-update}
Fix a step size $\gamma>0$.
If the NPG update is guided by the signal $A^\pi(\bs_i,a_i)+r_i(\bs_i,a_i)$, then
\begin{equation}
\begin{aligned}
\pi'(a_i\mid \bs_i)
&\propto\;
\pi(a_i\mid \bs_i)\,
\exp\!\Big(\gamma\big(A^\pi(\bs_i,a_i)+r_i(\bs_i,a_i)\big)\Big),
\\
Z^\pi(\bs_i)
&\triangleq\;
\sum_{a_i}{\pi(a_i\mid \bs_i)}\!
\left[\exp\!\Big(\gamma\big(A^\pi(\bs_i,a_i)+r_i(\bs_i,a_i)\big)\Big)\right],\\
&\quad\text{so that}\quad
\frac{\pi'(a_i\mid \bs_i)}{\pi(a_i\mid \bs_i)}=\frac{\exp\!\Big(\gamma\big(A^\pi(\bs_i,a_i)+r_i(\bs_i,a_i)\big)\Big)}{Z^\pi(\bs_i)}.
\end{aligned}
\end{equation}
\end{lemma}
See proof of Lemma F.2 in \citep{setlur2024rewarding}. Next, we restate Theorem \ref{theorem:improvement} in the following proposition.

\begin{proposition}[\textbf{Full-strength policy improvement lower bound}]
\label{prop:full-strength}
Let $\pi'$ be the NPG update in Lemma~\ref{lem:npg-update}. We can have:
\begin{equation}
\begin{aligned}
\label{eq:full-strength}
\E_{\bs_i\sim \rho}\!\left[V^{\pi'}(\bs_i)-V^{\pi}(\bs_i)\right]
\;\gtrsim\;
\E_{\bs_i\sim \rho}\!\Big[
\underbrace{\Var_{\pi}\!\big[r_{i,\mathrm{rea}}(\bs_i,a_i)\big]}_{\text{distinguishability (reasoning reward)}}
+\underbrace{\Var_{\pi}\!\big[r_{i,\mathrm{tab}}(\bs_i,a_i)\big]}_{\text{distinguishability (table reward)}}\\
+ 2\,\underbrace{\Cov_{\pi}\!\big(r_{i,\text{rea}}(\bs_i,a_i), \,r_{i,\text{tab}}(\bs_i,a_i)\big)}_{\text{alignment between $r_{i, \text{rea}}$ and $r_{i, \text{tab}}$}}
+ \underbrace{\E_{\pi}\!\big[r_{i,\text{tab}}(\bs_i,a_i)\,A^\pi(\bs_i,a_i)\big]}_{\text{alignment of $r_{i,\text{tab}}$ with $A^\pi$}
}
+ \underbrace{\E_{\pi}\!\big[r_{i,\text{rea}}(\bs_i,a_i)\,A^\pi(\bs_i,a_i)\big]}_{\text{alignment of $r_{i,\text{rea}}$ with $A^\pi$}
}
\Big].
\end{aligned}
\end{equation}
\end{proposition}

\begin{proof}[Proof of Proposition \ref{prop:full-strength}]
We now combine the performance difference lemma with the NPG update to derive a variance–alignment lower bound, while first retaining the covariance term between the reward components. By Lemma~\ref{lem:pdl}, we have
\begin{align}
\label{eq:pdl}
\E_{\bs_i\sim \rho}\!\big[V^{\pi'}(\bs_i)-V^{\pi}(\bs_i)\big]
\;=\;
\E_{\bs_i\sim d_{\rho}^{\pi'}} \E_{a_i\sim \pi'(\cdot\mid \bs_i)}\!\big[A^{\pi}(\bs_i,a_i)\big].
\end{align}

\textbf{Exponential tilting and a log-partition bound.} Let us define the log-partition at state $\bs_i$ by
\[
\log Z^\pi(\bs_i)
\;=\;
\log \E_{a_i\sim \pi(\cdot\mid \bs_i)}
\exp\!\Big(\gamma\big(A^\pi(\bs_i,a_i)+r_i(\bs_i,a_i)\big)\Big).
\]
From Lemma~\ref{lem:npg-update}, we have
\[
A^\pi(\bs_i,a_i)
=
\frac{1}{\gamma}\log\!\frac{\pi'(a_i\mid \bs_i)}{\pi(a_i\mid \bs_i)}
- r_i(\bs_i,a_i)
+ \frac{1}{\gamma}\log Z^\pi(\bs_i).
\]
Averaging over $a_i\sim \pi'(\cdot\mid \bs_i)$, using $\E_{\pi'}[\log\frac{\pi'}{\pi}]\ge 0$, Jensen’s inequality on $\log Z^\pi(\bs_i)$ and $\E_{\pi}[A^\pi(\bs_i,a_i)]=0$ gives
\begin{align}
\E_{a_i\sim \pi'(\cdot\mid \bs_i)}[A^\pi(\bs_i,a_i)]
\;\ge\;
-\,\E_{a_i\sim \pi'(\cdot\mid \bs_i)}[r_i(\bs_i,a_i)]
+ \E_{a_i\sim \pi(\cdot\mid \bs_i)}[r_i(\bs_i,a_i)].
\end{align}
Plugging this into \eqref{eq:pdl} yields the basic inner-product lower bound
\begin{align}
\label{eq:basic-ip}
\E_{\bs_i\sim \rho}\!\big[V^{\pi'}(\bs_i)-V^{\pi}(\bs_i)\big]
\;\ge\;
\E_{\bs_i\sim d_{\rho}^{\pi'}}\!
\left\langle \pi'(\cdot\mid \bs_i)-\pi(\cdot\mid \bs_i),\, r_i(\bs_i,\cdot)\right\rangle.
\end{align}

Using first-order expansion of the exponential tilt implies
\begin{align}
\label{eq:policy-move}
\left\langle \pi'(\cdot\mid \bs_i)-\pi(\cdot\mid \bs_i),\, r_i(\bs_i,\cdot)\right\rangle
\;\gtrsim\;
\left(
\Var_{\pi}\!\big[r_i(\bs_i,a_i)\big]
+\E_{\pi}\!\big[r_i(\bs_i,a_i)\,A^\pi(\bs_i,a_i)\big]
\right),
\end{align}
Combining \eqref{eq:basic-ip} and \eqref{eq:policy-move}, and weakening $d_{\rho}^{\pi'}$ to $\rho$ (componentwise monotonicity) gives
\begin{align}
\label{eq:var-plus-align}
\E_{\bs_i\sim \rho}\!\big[V^{\pi'}(\bs_i)-V^{\pi}(\bs_i)\big]
\;\gtrsim\;
\,\E_{\bs_i\sim \rho}\!\Big[
\Var_{\pi}\!\big[r_i(\bs_i,a_i)\big]
+\E_{\pi}\!\big[r_i(\bs_i,a_i)\,A^\pi(\bs_i,a_i)\big]
\Big].
\end{align}

\textbf{Variance decomposition with covariance.}
Next, using $r_i=r_{i,\text{rea}}+r_{i,\text{tab}}$, we have
\begin{align}
\Var_{\pi}\!\big[r_i(\bs_i,a_i)\big]
&= \Var_{\pi}\!\big[r_{i,\text{rea}}(\bs_i,a_i)\big]
 + \Var_{\pi}\!\big[r_{i,\text{tab}}(\bs_i,a_i)\big]
 + 2\,\Cov_{\pi}\!\big(r_{i,\text{rea}}(\bs_i,a_i),r_{i,\text{tab}}(\bs_i,a_i)\big).
\end{align}
Substituting into \eqref{eq:var-plus-align} complete our proof of Proposition \ref{prop:full-strength} (\eqref{eq:full-strength}).
\end{proof}

\paragraph{Covariance elimination under our reward design.}
By construction in our setup (see Section~\ref{sec:data_curation}), for each state–action pair $(\bs_i,a_i)$, the two components of the PRM signal, i.e., table reward and reasoning reward, are \emph{mutually exclusive}. Formally, we have 
\[
r_{i,\text{tab}}(\bs_i,a_i)\in\{-1,0,1\},\quad
r_{i,\text{rea}}(\bs_i,a_i)\in\{-1,0,1\},\quad
\text{and}\quad
r_{i,\text{tab}}(\bs_i,a_i)\,r_{i,\text{rea}}(\bs_i,a_i)=0.
\]
Policy-gradient updates are invariant to adding any per-state baseline, so we may center each component without loss, i.e., 
\[
\tilde r_{i,\text{rea}}(\bs_i,a_i)
\;=\; r_{i,\text{rea}}(\bs_i,a_i)-\E_{\pi}\!\big[r_{i,\text{rea}}(\bs_i,a_i)\big],
\qquad
\tilde r_{i,\text{tab}}(\bs_i,a_i)
\;=\; r_{i,\text{tab}}(\bs_i,a_i)-\E_{\pi}\!\big[r_{i,\text{tab}}(\bs_i,a_i)\big].
\]
Mutual exclusivity yields
$\E_{\pi}\!\big[\tilde r_{i,\text{rea}}(\bs_i,a_i)\,\tilde r_{i,\text{tab}}(\bs_i,a_i)\big]=0$,
hence
$\Cov_{\pi}\!\big(\tilde r_{i,\text{rea}},\tilde r_{i,\text{tab}}\big)=0$ and
\[
\Var_{\pi}\!\big[\tilde r_{i}(\bs_i,a_i)\big]
=
\Var_{\pi}\!\big[\tilde r_{i,\text{rea}}(\bs_i,a_i)\big]
+
\Var_{\pi}\!\big[\tilde r_{i,\text{tab}}(\bs_i,a_i)\big],
\quad
\tilde r_{i}\;\triangleq\;\tilde r_{i,\text{rea}}+\tilde r_{i,\text{tab}}.
\]
Plugging these centered quantities into the bounds of Proposition~\ref{prop:full-strength} (which is NPG-invariant under per-state centering) gives exactly Theorem \ref{theorem:improvement}'s inequality:
\begin{equation}
\begin{aligned}
\E_{\bs_i\sim \rho}\!\left[V^{\pi'}(\bs_i)-V^{\pi}(\bs_i)\right]
\; &\gtrsim\;
\E_{\bs_i\sim \rho}\!\Big[
\Var_{\pi}\!\big[r_{i,\text{rea}}(\bs_i,a_i)\big]
+\Var_{\pi}\!\big[r_{i,\text{tab}}(\bs_i,a_i)\big]
\\ & + \E_{\pi}\!\big[r_{i,\text{tab}}(\bs_i,a_i)\,A^\pi(\bs_i,a_i)\big]
+ \E_{\pi}\!\big[r_{i,\text{rea}}(\bs_i,a_i)\,A^\pi(\bs_i,a_i)\big]
\Big],
\end{aligned}
\end{equation}
which completes the proof of Theorem~\ref{theorem:improvement}.
\hfill$\square$

\vspace{10pt}
\begin{remark}
(i) Proposition~\ref{prop:full-strength} is strictly more general; Theorem~\ref{theorem:improvement} follows as a corollary under mutual exclusivity plus per-state centering (baseline invariance).
(ii) Mutual exclusivity alone yields $\E_{\pi}[r_{i,\text{rea}}\,r_{i,\text{tab}}]=0$, but per-state centering is what ensures $\Cov_{\pi}(r_{i,\text{rea}},r_{i,\text{tab}})=0$.
(iii) The alignment term necessarily uses the composite signal $r_i$ because the NPG step is guided by $A^\pi+r_i$.
\end{remark}

\newpage
\section{Experimental Setups}
\label{app:set_up}

\subsection{Policy Model Configurations}
\label{app:baseline_model}

In our experiments, we adopt an LRM DeepSeek-R1-Distill-Qwen-14B~\citep{deepseek_r1} as the downstream policy model. During inference, we configure the model with a temperature of 0.7, a maximum generation length of 16,384 tokens, and top-$p$ sampling with $p=0.95$. 
We evaluate the LRM on several inference-time scaling strategies:

\paragraph{Best-of-N (BoN).} 
The policy model generates $N$ candidate responses independently. A verifier (PRM) scores each response, and the final output is selected based on a voting or scoring method.

\paragraph{Beam Search.} 
Given beam width $N$ and branching factor $M$, the model generates $N$ initial steps. The verifier then selects the top $N/M$ continuations, and the model expands each with $M$ new candidates. This process repeats until termination, enabling guided exploration of high-quality reasoning paths.

\paragraph{Diverse Verifier Tree Search (DVTS).} 
DVTS is a variant of beam search where the search process is divided into multiple subtrees. Each subtree is explored independently using verifier-guided expansions, with candidates selected at every step based on PRM scores.

\paragraph{Majority Voting.} 
After generating multiple responses, the final answer is determined by simple majority over identical outputs, regardless of intermediate step scores. This method provides a baseline aggregation mechanism.

\paragraph{LLM-as-a-Judge.} 
Instead of relying solely on PRMs, a separate LLM is prompted to compare and evaluate candidate responses directly, selecting the most plausible or logically consistent output.

\subsection{Downstream Datasets}
\label{app:dataset}

\paragraph{TableBench \citep{tablebench}.} 
TableBench is a comprehensive benchmark specifically designed to evaluate the reasoning abilities of LLMs over tabular data. It consists of 3,681 unique tables drawn from diverse domains such as finance, sports, politics, and science, with each table containing on average $16.7$ rows and $6.7$ columns. The dataset emphasizes numerical reasoning, with over $65\%$ of table cells containing numerical values. TableBench questions are organized into four major categories: fact-checking, numerical reasoning, data analysis, further divided into 18 subcategories, yielding a total of 886 carefully annotated samples. Each question typically requires $6.3$ reasoning steps, making the dataset significantly more complex than prior TableQA corpora.

\paragraph{WikiTableQuestions (WTQ) \citep{wikitq}.}  
The WikiTableQuestions dataset introduces question answering over semi-structured HTML tables, aiming to test both compositional reasoning and domain generalization. It comprises $22{,}033$ natural language questions paired with $2{,}108$ Wikipedia tables, where the training and test tables are disjoint to ensure generalization to unseen schemas. The tables are semi-structured and heterogeneous, often containing multi-part cell values (e.g., ``Beijing, China'') that require normalization into multiple semantic types such as numbers or dates. Questions range from simple lookups to highly compositional queries involving comparison, aggregation, arithmetic, and superlatives. Each table contains at least $8$ rows and $5$ columns, and the question collection was conducted with quality control through multiple annotators.

\paragraph{MMQA \citep{wummqa}}  
MMQA is a large-scale benchmark for evaluating LLMs on multi-table and multi-hop question answering. The benchmark includes a total of $3{,}312$ relational tables across $138$ domains, where each instance consists of two or three interlinked tables. The dataset features $5{,}000$ multi-table samples, annotated with natural language questions, SQL queries, gold answers, and explicit primary/foreign key relations. To ensure annotation quality, foreign and primary keys were labeled by human experts with inter-annotator agreement exceeding $80\%$. MMQA questions span four main categories, including numerical, list, count, and select, with an average length of 77–85 tokens, reflecting their compositional complexity.

\subsection{\our{} Training Details}
\label{app:implementation}

We train \our{} using the off-the-shelf Qwen-3-8B model~\citep{qwen3} on our curated 60k dataset. For supervised fine-tuning, we adopt the LLaMA-Factory framework~\citep{zheng2024llamafactory}. The training setup uses a learning rate of $1\times10^{-5}$, a weight decay of $1\times10^{-4}$, a maximum sequence length of 20,000, and is run for 3 epochs. For the RL training stage, we adopt the VeRL framework~\citep{sheng2024hybridflow} to further optimize the SFT checkpoint via policy optimization. The model is trained with a batch size of 32, generating 8 samples per question as the group size, and is run for 3 epochs. During inference, we use the OpenR framework~\citep{wang2024openr} to deploy our trained \our{}-8B, which serves as a verifier to guide the downstream LRM under different test-time scaling strategies.

\section{Additoinal Experiments}
\label{app:add_exps}
\subsection{Ablation Study on \our{}}
\label{app:ablations}

\begin{table}[!t]
\centering
\begin{minipage}{0.49\linewidth}
    \centering
    \caption{Ablation on confidence calibration $\lambda_\text{cal}$.}
    \label{tab:ablation_param_1}
    \resizebox{0.75\linewidth}{!}{
    \begin{tabular}{lccc}
    \toprule
    N=32 & TB-NR & TB-FC & TB-DA \\
    \midrule
    0.3 &  76.8 & 80.9 & 33.1 \\
    0.5 &  77.3 & 81.3 & 33.6 \\
    0.8 & 78.1 & \textbf{82.0} & \textbf{34.3} \\
    1.0 & \textbf{78.5} & 81.4 & 33.8 \\
    \bottomrule
    \end{tabular}
    }
\end{minipage}
\hfill
\begin{minipage}{0.49\linewidth}
    \centering
    \caption{Ablation on tool-grounding $\lambda_\text{tool}$.}
    \label{tab:ablation_param_2}
    \resizebox{0.75\linewidth}{!}{
    \begin{tabular}{lccc}
    \toprule
    N=32 & TB-NR & TB-FC & TB-DA \\
    \midrule
    0.1 & 75.2 & 76.3 & 30.8 \\
    0.5 & 75.9 & 76.9 & 32.2 \\
    1.0 & \textbf{78.1} & \textbf{82.0} & 34.3 \\
    1.3 & 77.5 & 81.2 & \textbf{34.6} \\
    \bottomrule
    \end{tabular}
    }
\end{minipage}
\end{table}

\textbf{Ablations on $\lambda_\text{cal}$ and $\lambda_\text{tool}$.} In Eq.~\ref{eqa:rewar_design}, we use $\lambda_{\text{cal}}$ and $\lambda_{\text{tool}}$ as tunable coefficients to balance the contributions of the corresponding reward terms in GRPO. 
To examine their influence, we separately train our verifier model (initialized from the same SFT checkpoint) by varying $\lambda_{\text{cal}} \in \{0.3, 0.5, 0.8, 1.0\}$ and $\lambda_{\text{tool}} \in \{0.1, 0.5, 1.0, 1.5\}$ during RL, and then evaluate on TableBench with N = 32. As shown in Table~\ref{tab:ablation_param_1} and \ref{tab:ablation_param_2}, performance improves as $\lambda_{\text{cal}}$ increases, peaking at 0.8–1.0. For $\lambda_{\text{tool}}$, accuracy rises steadily and is strongest around 1.0–1.3.
These results empirically confirm the effectiveness of confidence calibration and tool-grounding in enhancing TTS.

\subsection{Performance Gain of \our{} with Increasing Number of Responses}

\begin{wrapfigure}{r}{0.38\textwidth}
\vspace{-15pt}
  \centering
    \includegraphics[width=0.95\linewidth]{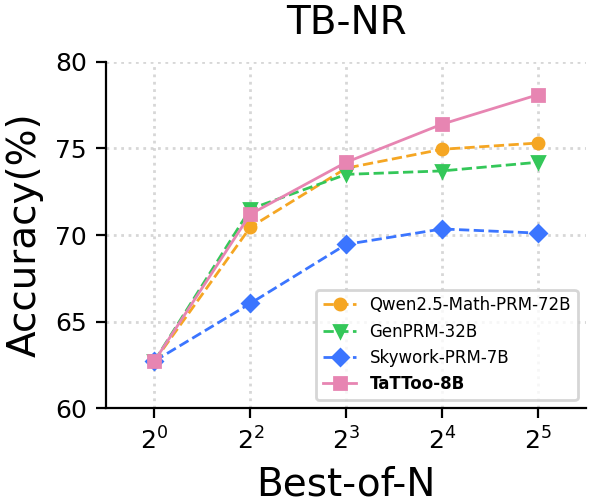}
    \caption{Performance of \our{} and baseline PRMs on TB-NR under Best-of-N test time scaling strategy. 
    }
    \label{fig:tb-nr}
    \vspace{-15pt}
\end{wrapfigure}
Figure~\ref{fig:tb-nr} presents the best-of-$N$ performance on TB-NR. We observe that baseline PRMs such as Qwen2.5-Math-PRM-72B and GenPRM-32B quickly saturate beyond $N{=}16$, achieving only marginal improvements at larger $N$. Skywork-PRM-7B shows even weaker scalability, plateauing below 71\%. 
In contrast, \our{} continues to improve steadily as $N$ increases, reaching 78.3\% at $N{=}32$, the highest among all models despite having significantly fewer parameters (8B). These results highlight the scalability advantage of \our{}, demonstrating its ability to leverage larger response pools more effectively than existing PRMs. 

In addition, we observe that the performance gap widens as $N$ grows: at $N{=}2$, the difference between \our{} and GenPRM-32B is less than 1\%, whereas by $N{=}32$ the gap exceeds 6\%. This suggests that \our{} is not only competitive at small-scale selection but also excels under test-time scaling, where robust discrimination among diverse reasoning trajectories becomes increasingly important. This trend highlights the robustness of our tool-augmented PRM in supervising multiple reasoning candidates.

\newpage 
\subsection{Case Study}
\label{app:case_study}

\begin{figure}[!t]
\begin{tcolorbox}[
  title={Case Study on \our{}},
  colframe=black!60,             %
  colback=white,              %
  colbacktitle=black!70,      %
  coltitle=white,             %
  fonttitle=\bfseries,
  width=\textwidth,
  left=0mm,
  right=0mm,
  boxrule=1pt,                %
  borderline={0.8pt}{0pt}{black}, %
  top=1mm,
  bottom=1mm,
  arc=1mm                     %
]

\footnotesize

\underline{\textbf{\normalsize(i) Before RL Training}} 
\vspace{5pt}

Okay, I need to verify step by step for the provided answer and determine if each step is correct or not.\\ 

\textbf{[Step 0 Evaluation]}\\
For step 0, let me first check whether it retrieves the correct information from the input table …
…

\textbf{[Step 3 Evaluation]}\\
Looking at step 3 and the table data provided, I need to verify whether the summarization calculation on the column "capacity in persons / hour” is correct. Let me add each part together:\\

350 + 500 + 970 + 10000 + 
2000 + 2800 + 650 + 650 + 
1000 + 1000 + 720 + 820 + 1000 = \textcolor{red}{16,920} 

\textcolor{red}{\textit{(wrong PRM calculation)}}\\

Since the provided step 3 gives a final calculation result of 22,460, which is not equal to 16,920. Step 3 is incorrect due to the wrong summarization on the overall carrying capacity. \\

The step is \textbf{\texttt{\textbackslash boxed\{-1(incorrect)\}}} \textcolor{red}{(\ding{55} \textit{Wrong Step Reward})}
\vspace{10pt}

\hrule height 1pt
\vspace{10pt}
\underline{\textbf{\normalsize(ii) After RL Training}} 
\vspace{5pt}

Alright, I am given a complete table, a question, and a corresponding answer. I need to check the correctness of each step of the answer. Let me verify step by step:\\

\textbf{[Step 0 Evaluation]}\\
I start by checking if Step 0 pulls the proper information from the table ...

...

\textbf{[Step 3 Evaluation]}\\
To verify Step 3, I need to check the summarization (sum) over the "capacity in persons / hour” column. I will write Python code to verify this using the input table.\\ 

\textbf{\textcolor{violet!60!black}{\texttt{Code Sandbox}}} \vspace{-5pt}
\begin{tcolorbox}[colback=violet!5!white, colframe=violet!60!black, boxrule=0.5pt, arc=.5mm, left=1mm, right=1mm, top=0mm, bottom=0mm, left skip=5pt, right skip=5pt,] \begin{lstlisting}[language=Python, basicstyle=\ttfamily\footnotesize, showstringspaces=false]
table = {
    "columns": ["name or route", "capacity in persons / hour"],
    "data": [
        ["betten betten village", 350],
        ["betten village bettmeralp", 500],
        # ...
        ["trainerlift", 820],
        ["lager 1", 1000],
        # ...
    ],
}
capacity_index = table["columns"].index("capacity in persons / hour")
total_capacity = sum(row[capacity_index] for row in table["data"])

print("Total capacity in persons/hour: ", total_capacity)
\end{lstlisting}
\end{tcolorbox}

\textbf{\textcolor{cyan!60!black}{\texttt{Code Interpreter}}}
\vspace{-5pt}
\begin{tcolorbox}[
  colback=cyan!5!white,
  colframe=cyan!60!black,
  boxrule=0.5pt,
  arc=.5mm,
  left=0.5mm,   %
  right=0.5mm,  %
  top=2mm,    %
  bottom=2mm, %
  left skip=5pt, 
  right skip=5pt,
]
\texttt{Total capacity in persons/hour: 22460}
\end{tcolorbox}

Since the computed total of 22,460 aligns with the result in step 3, this step is correct.\\

The step is \textbf{\texttt{\textbackslash boxed\{1(correct)\}}} \textcolor{impr}{(\ding{51} \textit{Correct Step Reward})}
\end{tcolorbox}
\caption{Case study on \our{} before and after RL training. Applying RL training enhances our PRM’s dynamic utilization of tool calls, which in turn provides more reliable supervision over the input reasoning trajectories of LRMs.}
\label{fig:case_study}
\end{figure}

As shown in Figure \ref{fig:case_study}, we present a case study on \our{} illustrating the difference between the verification processes at the two training stages on a specific instance in Figure~\ref{fig:case_study} (Appendix \ref{app:case_study}). 
When facing the same step (Step~3), the SFT-stage relies on inner text reasoning to verify the calculation, but introduces numerical errors that lead to incorrect justification of the step’s correctness. In contrast, the RL-stage learns to leverage the computation tool with concise Python code, ensuring accurate calculations and thereby providing more reliable reward supervision on the policy model’s responses.
In addition, we randomly sample 500 trajectories from both stages of \our{} on the same set of inputs and observe a 26.3\% improvement in the tool-integration ratio after RL training, indicating our model learns to utilize tools better for step-level verification during RL rollouts.

\section{Limitations and Broader Impacts}
\label{app:limitaion}
While reinforcement learning with reward shaping enhances our PRM’s ability to capture fine-grained tabular reasoning signals, it introduces more computational overhead. Compared to SFT-only training, the RL stage requires additional rollouts, reward evaluations, and optimization steps, which can increase training cost and resource demands. This overhead may hinder reproducibility and accessibility in low-resource environments, motivating future work on more efficient reward objectives and lightweight reward modeling strategies. In addition, our current framework is limited to text–table reasoning, and extending it to multimodal settings (e.g., integrating charts or image-based tables) remains an important direction for future work.

From a broader perspective, this work highlights the potential for process reward models to enhance structured reasoning in domains such as fact-checking, scientific analysis, and decision support. At the same time, reliance on automated verification carries risks: if tools or training data contain errors, these may be amplified rather than corrected. We encourage future research to explore mechanisms for auditing verifier reliability, reducing the energy footprint of RL training, and ensuring equitable performance across diverse application domains.

\end{document}